%% file: paper_arxiv.tex
\documentclass{article}
\usepackage{ijcai25}

\usepackage{times} 
\usepackage{soul}
\usepackage{url}
\usepackage[hidelinks]{hyperref}
\usepackage[utf8]{inputenc}
\usepackage[small]{caption}
\usepackage{graphicx}
\usepackage{amsmath}
\usepackage{amssymb} 
\usepackage{bbm}
\usepackage{amsthm}
\usepackage{booktabs}
\usepackage{algorithm}
\usepackage{algorithmic}
\usepackage[switch]{lineno}
\usepackage{mdframed}
\usepackage{xcolor}
\usepackage{multirow}
\usepackage{booktabs}
\usepackage{tikz}

\usepackage{listings}
\usepackage{xcolor}
\usepackage{multicol}

\lstdefinestyle{python}{
    language=Python,
    basicstyle=\ttfamily\footnotesize,
    keywordstyle=\color{blue}\bfseries,
    stringstyle=\color{red},
    commentstyle=\color{green!50!black},
    numberstyle=\tiny,
    stepnumber=1,
    numbersep=5pt,
    frame=lines,
    breaklines=true,
    showstringspaces=false,
    captionpos=b,
    tabsize=4,
    escapeinside={(*@}{@*)}
}

\usepackage{scalerel}
\def\imgtwo{\scalerel*{\includegraphics{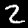}}{X}\hspace{0.5em}} 
\def\imgthree{\scalerel*{\includegraphics{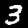}}{X}\hspace{0.5em}} 
\def\imgfive{\scalerel*{\includegraphics{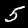}}{X}\hspace{0.5em}} 
\def\imgeight{\scalerel*{\includegraphics{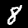}}{X}\hspace{0.5em}}

\newtheorem{example}{Example}
\newtheorem{theorem}{Theorem}

\input{markup}

\title{On Scaling Neurosymbolic Programming through Guided Logical Inference}

\author{
Thomas Jean-Michel Valentin$^{1,2}$ \and
Luisa Sophie Werner$^2$ \and
Pierre Gen\`eves$^2$ \and 
Nabil Laya\"ida$^2$ \\\
\affiliations
$^1$ENS Paris-Saclay, 4 av des Sciences, 91190 Gif-sur-Yvette, France, \url{https://ens-paris-saclay.fr/}\\
$^2$TYREX Univ. Grenoble Alpes, CNRS, Inria,
Grenoble INP, LIG,  655 av Europe, 38330 Montbonnot Saint Martin, France, \url{https://www.inria.fr/fr/tyrex}\\
\emails
thomas.valentin@ens-paris-saclay.fr,
\{Luisa.Werner, Pierre.Geneves\}@inria.fr
}

\begin{document}

\maketitle

\begin{abstract}
Probabilistic neurosymbolic learning seeks to integrate neural networks with symbolic programming. 
Many state-of-the-art systems rely on a reduction to the Probabilistic Weighted Model Counting Problem (PWMC), which requires computing a Boolean formula called the logical provenance.
However, PWMC is \#P-hard, and the number of clauses in the logical provenance formula can grow exponentially, creating a major bottleneck that significantly limits the applicability of PNL solutions in practice.
We propose a new approach centered around an exact algorithm DPNL, that enables bypassing the computation of the logical provenance.
The DPNL approach relies on the principles of an oracle and a recursive DPLL-like decomposition in order to guide and speed up logical inference.
Furthermore, we show that this approach can be adapted for approximate reasoning with $\epsilon$ or $(\epsilon, \delta)$ guarantees, called ApproxDPNL.
Experiments show significant performance gains.
DPNL enables scaling exact inference further, resulting in more accurate models.
Further, ApproxDPNL shows potential for advancing the scalability of neurosymbolic programming by incorporating approximations even further, while simultaneously ensuring guarantees for the reasoning process.
\end{abstract}

\section{Introduction}
Neurosymbolic artificial intelligence \cite{NeSy:Survey:vanHarmelen2019,NeSy:Survey:Varmeulen2023} seeks to integrate neural perception with symbolic methods to produce more accurate, efficient, and reliable models. 

Probabilistic Neurosymbolic Learning is a research direction that combines neural networks with probabilistic reasoning \cite{NeSy:DeepProbLog:Manhaeve2018,NeSy:Survey:DeRaedt2020,NeSy:Scallop:Siyang2023}. 
Neural networks are employed to extract probability distributions from raw data.
They are end-to-end integrated with the reasoning process of a logic program to derive solutions.
In contrast to pure neural black-box models, the passage through a logic program provides a clear lineage of how a result is derived, enhancing interpretability. 
 
However, neurosymbolic models face significant scalability challenges due to the complexity of the inference in the logic program.  
Prominent systems for exact inference such as DeepProbLog \cite{NeSy:DeepProbLog:Manhaeve2018} rely on a reduction to the Probabilistic Weighted Model Counting problem (PWMC). 
This requires computing a logical provenance formula, whose size can grow exponentially depending on the task. 
Furthermore, solving the PWMC is \#P-hard. 
This significantly limits the applicability of PNL approaches in practice, as they struggle to handle large problem instances effectively.

Approximations have been developed to overcome these scalability issues \cite{NeSy:DeepProbLog:Manhaeve2021,NeSy:DeepStochLog:Winters2022,NeSy:A-NeSI:vanKrieken2023,NeSy:Scallop:Huang2021}. 
However, these approximations do not provide guarantees, which potentially undermines one of the primary benefit of PNL: increased confidence in the reliability of the model.

\paragraph{Contribution.}
We introduce a new approach for probabilistic neurosymbolic learning based on a new algorithm, DPNL, for exact probabilistic reasoning.
One of its main advantages is that it bypasses the explicit computation of the logical provenance formula.
Instead, DPNL relies on an oracle to guide and accelerate the logical inference process.
This is achieved by recursively decomposing probabilistic computations in a DPLL-inspired manner.
We explain how to construct valid oracles and formally prove the termination and correctness of DPNL.
Additionally, it can be adapted for approximate computations with $\epsilon$ or $(\epsilon, \delta)$ guarantees.
Experiments demonstrate that DPNL scales better in practice. 
This makes it possible to build accurate models for tasks that previous PNL methods with exact inference cannot handle.
Furthermore, the experiments show that ApproxDPNL is advantageous in scaling DPNL to even larger tasks while providing reliability in the form of guarantees for the approximation.

\section{Problem statement and definitions\label{Section:FormalDefinition}}
\paragraph{Notations.}
We first introduce some notations.
A function $f$ that maps elements from a set $A$ to a set $B$ is denoted as $f : A \rightarrow B$.
The set of all such functions is written as $\Functions{A}{B}$.
For any function $f : A \rightarrow B$ and any $y \in B$, the preimage of $y$, i.e., the set of all $x \in A$ such that $f(x)=y$, is denoted by $f^{-1}(y)$. An estimator of a function $f$ is denoted by $\hat{f}$.

We consider the Probabilistic Learning (PL) problem and focus on solving it using Probabilistic Neurosymbolic Learning (PNL) systems, which integrate both neural and symbolic components. The formal definitions follow.

\paragraph{Probabilistic Learning Problem.}
A PL problem is a quadruplet $P = ((\Omega, \tribe, \prob), \inputspace, \outputspace, \function)$ where:
\begin{itemize}
    \item $(\Omega, \tribe, \prob)$ is a probability space with sample space $\Omega$, event space $\tribe$ and probability function $\prob$. 
    \item $\inputspace$ is the input space.
    \item $\outputspace$ is the finite output set.
    \item $\function : \inputspace \rightarrow \Functions{\Omega}{\outputspace}$ is a function that maps each input $i \in \inputspace$ to its output random variable $\function(i)$.
\end{itemize}
The probability distribution of $P$ is the function $p$ that maps every input $i$ to the probability distribution of $F(i)$: \begin{center}
    $\forall i \in \inputspace$, $\distrib(i) : o \in \outputspace \mapsto \prob(F(i)=o)$.
\end{center}
We define solving $P$ as estimating $\distrib$.

\begin{example}[MNIST Classification]
\label{ex:mnist_classification}
Consider the MNIST classification task. 
The goal is to recognise handwritten digits in images, e.g. the digit in \imgeight is 8. 
This task can be modelled as the MNIST PL problem 
$((\Omega, \tribe, \prob), \inputspace_{\text{MNIST}}, \outputspace_{\text{MNIST}}, \function_{\text{MNIST}})$, where:
\begin{itemize}
    \item $\inputspace_{\text{MNIST}} = \mathbb{R}^{28\times 28}$ represents the space of all possible images,
    \item $\outputspace_{\text{MNIST}} = [0,9]$ is the set of possible output labels,
    \item $\function_{\text{MNIST}}$ associates with each image $i$ the random variable $\function_{\text{MNIST}}(i)$ for the digit in $i$.
\end{itemize}
Solving the MNIST PL problem amounts to estimating the function $\distrib_{\text{MNIST}}$ that maps each image $i$ to the probability distribution of $\function_{\text{MNIST}}(i)$, which can be done using a neural network $\hat{\distrib}_{\text{MNIST}}$.
\end{example}

\paragraph{Probabilistic Neurosymbolic Learning System.}
Let $P = ((\Omega, \tribe, \prob), \inputspace, \outputspace, \function)$ be a PL problem and $\distrib$ its corresponding probability distribution.
A PNL system for $P$ is a quadruplet $((\Vset{k})_{1\leq k \leq m}, (\Xvar{k})_{1 \leq k \leq m}, (\hat{\distrib_k})_{1\leq k \leq m}, \Sfunction)$ where:
\begin{itemize}
    \item $m \geq 1$ is the order of the PNL system.
    \item $\forall k \in [1,m]$, $P_k = ((\Omega, \tribe, \prob), \inputspace, \Vset{k}, \Xvar{k})$ is a PL problem with its probability distribution $\distrib_k$.
    \item $\forall i \in \inputspace$, the random variables $(\Xvar{k}(i))_{1\leq k \leq m}$ are independent. We say that the sub-problems $(P_k)_{1\leq k \leq m}$ are independent. 
    \item $\forall k \in [1,m]$, $\hat{\distrib_k}$ is a neural network which solves $P_k$, i.e, it estimates $\distrib_k$. 
    The set of parameters of $\hat{\distrib_k}$ is $\theta_k$.
   
    \item $S : \Vset{1} \times ... \times \Vset{m} \rightarrow \outputspace$ is a symbolic function that maps values obtained from the sub-problems to the output.
    $S$ is known.
    
    \item $\function = \Sfunction \circ (\Xvar{1},...,\Xvar{m})$, that is, for all $i \in \inputspace$ and $w\in \Omega$, $F(i)(w) = S(\Xvar{1}(i)(w),...,\Xvar{m}(i)(w))$.
\end{itemize}
The PNL system decomposes $P$ into $m$ smaller, independent sub-problems $(P_k)_{1 \leq k \leq m}$. Each sub-problem $P_k$ is solved by estimating $\distrib_k$ with $\hat{\distrib_k}$.
The outputs of $\hat{\distrib_k}$ are then combined using the symbolic function $\Sfunction$ to estimate $\distrib$ thanks to the following relation, which holds because of the independence assumption of the random variables:
\begin{align}
    p(i)(o) & = \sum_{(x_1,...,x_m)\in \Sfunction^{-1}(o)} \prob(\bigcap_{k=1}^m \Xvar{k}(i) = x_k) \label{eq:theory:core0}\\
            & = \sum_{(x_1,...,x_m)\in \Sfunction^{-1}(o)} \prod_{k=1}^m \distrib_k(i)(x_k). \label{eq:theory:core1}
\end{align}
The corresponding estimator $\hat{p}$ is defined by:
\begin{align}\label{eq:core}
    \hat{p}(i)(o) & = \sum_{(x_1,...,x_m)\in \Sfunction^{-1}(o)} \prod_{k=1}^m \hat{\distrib_k}(i)(x_k).
\end{align}
Moreover, $\hat{\distrib}(i)(o)$ is differentiable w.r.t. to $\theta=(\theta_k)_{1\leq k \leq m}$, which allows to train the $(\hat{\distrib_k})_{1\leq k \leq m}$ directly with gradient descent and a data set for the problem $P$.

\paragraph{Major Challenge.}
The primary computational challenge lies in efficiently evaluating the sum over $\Sfunction^{-1}(o)$ in equation (\ref{eq:core}).
Indeed, $|S^{-1}(o)|$ can grow exponentially w.r.t. $m$, and inverting a symbolic function is in essence a NP-hard task. 
Therefore, addressing the complexity of computing $S^{-1}(o)$ is a key focus in improving PNL system scalability.

\begin{example}[MNIST-SUM Task]
    \label{ex:mnist_sum}
We consider the neurosymbolic MNIST-SUM task \cite{NeSy:DeepProbLog:Manhaeve2018}, which will be used as a running example.
The goal is to learn to predict the sum of handwritten digits, given only the label of the sum, but not the labels of the individual digits in the images, e.g. \imgthree $+$ \imgfive =  8. 
This involves two subtasks: (1) perception of the digit in the image and (2) reasoning on the combination of the digits (with the sum as the symbolic function $S$). 
Building on the MNIST PL problem, see Example~\ref{ex:mnist_classification}, we define the MNIST-SUM PL problem as $((\Omega, \tribe, \prob), \inputspace, \outputspace, \function)$, where
\begin{itemize}
    \item $\inputspace = \inputspace_{\text{MNIST}} \times \inputspace_{\text{MNIST}}$ are pairs of images,
    \item $\outputspace = [0,18]$ are the output values of the sum,
    \item $F(i_1,i_2)= F_{\text{MNIST}}(i_1) + F_{\text{MNIST}}(i_2)$ is the random variable related to the sum of the digits in the two images $i_1, i_2$.
\end{itemize}

For this task, we also define the PNL system $((\Vset{1}, \Vset{2}), (\Xvar{1}, \Xvar{2}),(\hat{\distrib_{1}}, \hat{\distrib_{2}}), S)$, where
\begin{itemize}
    \item $\Vset{1} = \Vset{2} = [0,9]$,
    \item $\Xvar{1}(i_1,i_2) = F_{\text{MNIST}}(i_1)$ and $\Xvar{2}(i_1,i_2) = F_{\text{MNIST}}(i_2)$ for all pairs of images $(i_1, i_2) \in \inputspace_{\text{MNIST}} \times \inputspace_{\text{MNIST}}$,

    \item $\hat{\distrib}_1(i_1,i_2) = \hat{\distrib}_{\text{MNIST}}(i_1)$ and $\hat{\distrib}_2(i_1,i_2) = \hat{\distrib}_{\text{MNIST}}(i_2)$ for all $(i_1,i_2) \in \inputspace^2$,
    \item $S(d_1,d_2) = d_1+d_2$ for all digit values $d_1 \in V_1, d_2 \in V_2$. 
\end{itemize}
Thus, this PNL system for MNIST-SUM consists of two independent MNIST PL problems for classifying the digits in the images. 

This task can easily be extended to MNIST-N-SUM, where N is the number of digits per summand. 
In this notation, the previous example is a MNIST-1-SUM. 
For N=2, two digits form the summands, e.g. \imgtwo \imgeight $+$ \imgthree \imgfive $=$ 63.
By extending from one-digit sums (N=1) to two-digits sums (N=2), the complexity of the task increases exponentially.
Thus, the multi-digit MNIST-N-SUM task is well-suited for exploring the ability  of PNL systems to scale.
\end{example} 

\begin{example}[Inverting $S$ for MNIST-N-SUM]
    The challenge of PNL is clearly to compute equation~\ref{eq:core} efficiently.
    The order of the PNL system for MNIST-N-SUM is $m = 2N$, which corresponds to the total number of digits to recognize ($N$ for the first summand and $N$ for the second).
    To compute $S^{-1}(o)$ for a given sum $o$, one could enumerate all possible combinations of digits $(d_1, \dots, d_{2N})$ and check if $S(d_1, \dots, d_{2N}) = o$.
    For example, given the MNIST-1-SUM task and the output sum value 4, all combinations of digits that sum to 4 are considered: $(0,4), (3,1) \hdots (4,0)$.
    This naive approach would require at least $10^{2N}$ operations.
\end{example}

A common approach, popularized by the seminal work on Problog~\cite{NeSy:ProbLog:DeRaedt2007}, is to consider that $S$ is expressed in a probabilistic Prolog language (see Appendix~\ref{appendix:applicationtodeepproblog} for details).
The probability of a query is then computed by reducing it to the probabilistic weighted model counting problem, a specific case of the weighted model counting problem whose formal definition is provided below.

\paragraph{Probabilistic Weighted Model Counting (PWMC).}
\label{definition:PWMC}
Consider:
\begin{itemize}
    \item $\mathbf{X}$ is a set of (independent) binary (random) variables.
    \item $G : \Functions{\mathbf{X}}{\{0,1\}} \rightarrow \{0, 1\}$ is a Boolean formula.
    \item $\sigma : \mathbf{X} \rightarrow [0,1]$ is a probability distribution.
\end{itemize}
For any valuation $\psi \in \Functions{\mathbf{X}}{\{0,1\}}$, we define the weight of $\psi$ w.r.t. $\sigma$ by
\begin{align*}
    w_\psi^\sigma := & \prod_{X \in \psi^{-1}(1)}\sigma(X)\cdot\prod_{X \in \psi^{-1}(0)}(1-\sigma(X)).
\end{align*}
A valuation $\psi \in \Functions{\mathbf{X}}{\{0,1\}}$ such that $G(\psi) = 1$ is called a model of $G$ and $G^{-1}(1)$ denotes the set of models of $G$.
The probabilistic weighted model count of $G$ under $\sigma$ is defined as
\begin{align}
    \text{PWMC}_G^\sigma := & \sum_{\psi \in G^{-1}(1)}w_\psi^\sigma.
\end{align}

The reduction to PWMC involves representing the symbolic function $S$ as a Boolean logical formula $G$, known as the logical provenance. 
Many PNL systems based on this reduction suffer from performance bottlenecks, in particular when explicitly processing the full logical provenance.

In the following section, we introduce a novel approach designed to enable the  construction of a new class of more efficient PNL systems.
This approach aims at avoiding key bottlenecks associated with the intermediate representation $G$.
It directly addresses the challenge of efficiently evaluating the sum over $\Sfunction^{-1}(o)$ in equation (\ref{eq:core}), using a novel method inspired by DPLL.

\section{A novel architecture for PNL systems}
We introduce the new approach in two steps: first, we present ProbDPLL as an intermediate algorithm that serves as the basis for key observations. Then, building on these observations, we extend it to a more general setting and introduce the DPNL algorithm.

\subsection{ProbDPLL}
We first introduce a new algorithm: ProbDPLL for computing $\text{PWMC}_G^\sigma$ when $G$ is a Boolean formula in CNF form.
ProbDPLL builds on the observation that model counting is analogous to PWMC. 
It extends the \#DPLL model counter introduced in \cite{WMC:DPLL:Bacchus2003}, which is based on the seminal DPLL algorithm \cite{WMC:DPLL:Davis1962}, adapting it to compute $\text{PWMC}_G^\sigma$. 

\paragraph{The ProbDPLL Algorithm.}
The pseudocode of ProbDPLL is given in Algorithm~\ref{alg:ProbDPLL}.
It uses the notation $G_{|X=b}$ to denote the CNF formula obtained by replacing $X$ with $b$ in $G$ and simplifying the formula.
It explores partial valuations of the formula $G$ and stops exploration when the partial valuation evaluates $G$ as true (when no clauses remain) or false (when the empty clause is reached).
A vizualization of an example run of ProbDPLL is given in Appendix~\ref{appendix:probdpllviz}.

\paragraph{Correctness and termination.}
For all $G$ provided in CNF and a probability distribution $\sigma$ about the variables in $G$, $\text{ProbDPLL}(G, \sigma)$ always terminates and returns $\text{PWMC}_G^\sigma$.
A detailed proof of the termination and correctness of the algorithm is provided in Appendix~\ref{appendix:proof:ProbDPLLCorrectness}.

\paragraph{Application to logic programs.}
We can apply this algorithm to probabilistic logic programs such as in DeepProbLog \cite{NeSy:DeepProbLog:Manhaeve2018} (c.f. Section~\ref{section:ObtainingOracleFromLogic}). 
Since the logical provenance in DeepProbLog is expressed in DNF, we instead apply the algorithm to its negation, which is in CNF. 
This yields a probability $p$, allowing us to return $1-p$ as the probability of the logical provenance.
 
\begin{algorithm}[tb]
    \caption{ProbDPLL($G, \sigma$)}
    \label{alg:ProbDPLL}
    \textbf{Input}: $G$ a Boolean formula represented in CNF, $\sigma$ a probability distribution of the variables in $\mathbf{X}$\\
    \textbf{Output}: $\text{PWMC}_G^\sigma$
    \begin{algorithmic}[1] 
        \IF {$G$ has no clauses}
            \STATE \textbf{return} $1$
        \ELSIF {$G$ contains the empty clause}
            \STATE \textbf{return} $0$
        \ELSE
            \STATE Choose $X\in \mathbf{X}$ such that $X$ appears in $G$
            \STATE $p_{G_{|X=1}} \gets \text{ProbDPLL}(G_{|X=1},\sigma)$
            \STATE $p_{G_{|X=0}} \gets \text{ProbDPLL}(G_{|X=0},\sigma)$
            \STATE \textbf{return} $\sigma(X)\cdot p_{G_{|X=1}} + (1-\sigma(X))\cdot p_{G_{|X=0}}$
        \ENDIF
    \end{algorithmic}
\end{algorithm}

\paragraph{Key observations.}
While improvements to ProbDPLL, such as the component caching introduced for \#DPLL by \cite{WMC:DPLL:Bacchus2003}, are possible, they do not directly address the issue of logical provenance size.
To tackle this problem, we highlight the following key observations:
\begin{itemize}

    \item ProbDPLL is designed for Boolean formulas but can be adapted to general (non-Boolean) functions, such as the symbolic function $S$ defined in Section~\ref{Section:FormalDefinition}, by working with independent random variables with finite domains (instead of independent binary random variables, as in PWMC).
    For instance, in the MNIST-SUM example~\ref{ex:mnist_sum}, the input values of $S$ are the digit values $[0,9]$.
    Therefore, it is more convenient to use random variables with values in the range $[0,9]$ than to encode them with independent binary random variables (which would result in an exponential blowup, as done in \cite{NeSy:DeepProbLog:Manhaeve2018} and detailed in Appendix~\ref{appendix:nad}). 
    ProbDPLL does not require $G$ to be in CNF.
    As long as $G$ is represented in a form that allows the algorithm to partially evaluate it with respect to a variable $X$ and determine whether this partial evaluation makes $G$ true or false, the algorithm will function correctly. 
    Furthermore, $G$ does not even need to be explicitly known.
    ProbDPLL can operate as long as it has access to the answers to the question: ``Does the partial valuation evaluate $G$ to true or false?''
    
\end{itemize}

\subsection{The DPNL Algorithm}
We now introduce the main contribution: the DPNL algorithm, which generalizes ProbDPLL to support more general symbolic functions $S$ (not necessarily Boolean formulas like $G$) and independent random variables with finite domains (instead of only independent binary random variables).

Instead of working directly with a specific representation of $S$, as ProbDPLL does with $G$, DPNL follows a different approach.
It calls an oracle function that extracts only the necessary information about $S$ to guide the inference process in a way similar to ProbDPLL.
The concept of an oracle here is an abstraction that captures precisely the properties required for DPNL to be correct.
The advantage of this method is that, in many cases, the oracle can be thoughtfully implemented, making DPNL more efficient than other PNL frameworks, such as DeepProbLog.

We now formally define the concepts of valuations and oracles used in DPNL.
 
\paragraph{Valuation.} 
A valuation $v$ of the variables $(X_k)_{1\leq k \leq m}$ is an array of size $m$ such that for all $k\in [1,m]$, $v[k] \in V_k \cup \{ \unknown \}$.
The following notations are used: \begin{itemize}
    \item $\setofvaluations{}$ denotes the set of all valuations.
    \item A \emph{total valuation} is one where : $\forall k \in [1,m]$, $v[k] \neq \unknown$.
    \item A \emph{partial valuation} is one that is not \emph{total}.
    \item $v'$ is a \emph{sub-valuation} of $v$ if: \begin{center}
        $\forall k\in[1,m]$, $v[k]\neq \unknown \implies v'[k]=v[k]$
    \end{center}
    \item $sub(v)$ is the set of sub-valuations of $v$.
    \item $tot(v)$ is the set of total sub-valuations of $v$.
    \item $w(i,v)$ is the event corresponding to $v$ with input $i$. Formally, it is the set of $w\in \Omega$ such that $[X_1(i)(w),...,X_m(i)(w)] \in tot(v)$.
    \item $v_{v[k]\leftarrow x}$ is the valuation $v$ where the value of $v[k]$ is replaced by $x$.
\end{itemize}

\paragraph{Oracle.} 
\label{section:oracle}
Given the symbolic function $S$  of a PNL system, an oracle $O_S$ for $S$ is a function that, given a valuation $v\in \setofvaluations{}$ and an output $o \in \outputspace$, answers the question
``Is $S$ always equal to or always different from $o$ for all sub-valuations of $v$?''.
Intuitively, from a given output $o$, the oracle is in charge of determining whether it is still relevant to explore sub-valuations or whether the exploration process can be stopped.

Formally, a function $O_S : \setofvaluations{} \times \outputspace \rightarrow \{0,1, \unknown\}$ is a valid oracle for $S$ if and only if, for all $v \in \setofvaluations{}$ and $o \in \outputspace$:
\begin{itemize}
    \item If the valuation $v$ is total,
    \begin{center}
        $O_S(v,o) = \begin{cases}
        1 \text{ if } S(v[1],...,v[m]) = o, \\
        0 \text{ otherwise }
    \end{cases}$
    \end{center}
    \item If $O_S(v,o) = 1$,
    \begin{center}
        $\forall v' \in tot(v)$, $S(v'[1], ..., v'[m]) = o$
    \end{center}
    \item If $O_S(v,o) = 0$,
    \begin{center}
        $\forall v' \in tot(v)$, $S(v'[1], ..., v'[m]) \neq o$
    \end{center}
\end{itemize}
We say that the oracle is complete if the last two points are equivalences.

\paragraph{DPNL.}
The pseudocode of DPNL is given in Algorithm~\ref{alg:DPNL}.
We can compute $\distrib(i)(o)$ by initially calling $\text{DPNL}(i,o,[\unknown,...,\unknown],O_S)$.
In practice, however, we want to compute $\hat{\distrib}(i)(o)$ because we do not know $\prob$.
Therefore, we replace $\prob(\Xvar{k}(i)=y)$ at line~\ref{alg:DPNL:useprob} by its neural estimation $\hat{\distrib_k}(i)(y)$.

\paragraph{Termination and Correctness.}
For every \begin{itemize}
    \item PL problem $\bigl((\Omega, \tribe, \prob), \inputspace, \outputspace, F\bigr)$,
    \item PNL system $\bigl((\Vset{k})_{1\leq k \leq m}, (\Xvar{k})_{1\leq k \leq m}, (\hat{\distrib_{k}})_{1\leq k \leq m}, S\bigr)$ for the PL problem,
    \item valid oracle $O_S$ for $S$,
    \item $i \in \inputspace$, $o \in \outputspace$ and $v \in \setofvaluations{}$,
\end{itemize}  
$\text{DPNL}(i,o,v,O_S)$ always terminates and returns $\prob(F(i) = o | w(i,v))$.
A detailed proof of the termination and correctness of the algorithm is provided in Appendix~\ref{appendix:proof:DPNLCorrectness}.

\begin{algorithm}[tb]
    \caption{DPNL($i,o,v,O_S$)}
    \label{alg:DPNL}
    \textbf{Input}: $i \in \inputspace$, $o \in \outputspace$, $v \in \setofvaluations{}$, $O_S$ an oracle for $S$ \\
    \textbf{Output}: $\prob(F(i) = o | w(i,v))$
    \begin{algorithmic}[1] 
        \IF {$O_S(v,o) = 1$}
            \STATE \textbf{return} $1$
        \ELSIF {$O_S(v,o) = 0$}
            \STATE \textbf{return} $0$
        \ELSE
            \STATE Choose $k \in [1,m]$ such that $v[k] = \unknown$. \label{alg:DPNL:choosek}
            \STATE \textbf{return} $\underset{y \in V_k}{\Sigma}\prob(X_k(i) = y)\label{alg:DPNL:useprob}\cdot\text{DPNL}(i,o,v_{v[k]\leftarrow y},O_S)$ 
        \ENDIF
    \end{algorithmic}
\end{algorithm}

\subsubsection{Advantages of using Oracles}

For any symbolic function $S$, if we dispose of a procedure computing $S$, then there exists a valid oracle for $S$.
Indeed, the oracle that returns $\unknown$ when the valuation is partial and tests if the result of $S$ is equal to the output when the valuation is total, is valid for $S$.
However, with this naive oracle, the complexity of DPNL is the same as testing $S$ over all possible inputs to compute $S^{-1}(o)$.
It is thus important to design oracles that return $0$ or $1$ for partial valuations in order to prune the search space of DPNL.
Again, a naive oracle could test all total sub-valuations of $v$ to answer $0$ or $1$ as soon as possible, but this is not efficient.
Thus, a good oracle should strike a balance between pruning capability and efficiency.

In practice, $S$ could naturally be represented by an algorithm.
It is often possible to derive an efficient valid oracle $O_S$ with pruning capability from this algorithm.
For example, in the MNIST-N-SUM task, $S$ is the sum algorithm which operates digit by digit, computing the digits of the result from right to left (see Algorithm~\ref{alg:Addition} in Appendix~\ref{appendix:MNISTAddOracle}).
Since this algorithm constructs the result of $S$ digit by digit, we can take advantage of it to design an efficient oracle which is able to prune the search space (see Algorithm~\ref{alg:AdditionOracle} in Appendix~\ref{appendix:MNISTAddOracle}).

This is a major advantage of the DPNL approach: it can readily work with naive oracles, while providing the flexibility to define custom oracles, allowing as much refinement and optimization as desired, depending on the effort invested in their design.

In the next section, we explain how valid oracles can be systematically derived from more general programs.

\section{Obtaining Oracles from Logic Programs}\label{section:ObtainingOracleFromLogic}

We now explain how valid oracles can be systematically derived from a broad class of logic programs. 
To this end, we first recall the connection between logic and probabilistic learning (PL) problems. Next, we define a class of neural probabilistic logic programs, building on ProbLog \cite{NeSy:ProbLog:DeRaedt2007} and DeepProbLog.
Finally, we provide an algorithm to compute oracles from these programs, and prove that the obtained oracles are valid.
The formal definitions follow.

\newcommand{\Logic}[0]{\mathcal{L}}
\newcommand{\Func}[0]{\mathcal{F}}
\newcommand{\Pred}[0]{\mathcal{P}}

\paragraph{Logics.}
A logic $\Logic$ consists of a set of functions $\Func$ and predicates $\Pred$, logical operators such as $\land$ and $\exists$, their semantics, and an infinite set of variables $\mathcal{X}$.
Combinations of functions and variables form terms, predicates applied to terms form atoms and logical combinations of atoms form logical formulas.
The purpose of a logic is to model aspects of reality, which is formalized through an interpretation $I = (\Delta, [.]_\Func, [.]_\Pred)$, where:
\begin{itemize}
    \item $\Delta$ is a set of real objects called the domain,
    \item for all $f\in \Func$, $[f]_\Func \in \Functions{\Delta^{arity(f)}}{\Delta}$ is the interpretation of $f$,
    \item for all $p\in \Pred$, $[p]_\Pred \in \Functions{\Delta^{arity(p)}}{\{0,1\}}$ is the interpretation of $p$.
\end{itemize}
This defines a semantics $[.]_I$ by recursively applying the semantics of logical operators and the interpretations of functions and predicates.

\begin{example}[Logic for MNIST]
    In MNIST, we can consider $\Logic$ where $\Func = \{i_1(0), i_2(0), \underline{0}(0), ..., \underline{9}(0)\}$ and $\Pred = \{is(2) \}$ where $s(k)$ denotes a symbol $s$ of arity $k$.
    We can define $I = (\text{Images} \cup [0,9], [.]_\Func, [.]_\Pred)$ where $[i_1]_\Func, [i_2]_\Func$ are images, $[\underline{d}]_\Func = d$ and where, for example, $[is(i_1,\underline{7})]_I = [is]_\Pred([i_1]_\Func, [\underline{7}]_\Func) = 1$ iff $[i_1]_\Func$ is an image of $7$.
\end{example}

In the MNIST classification example, the task consists in estimating the probability of $is(i_1, \underline{d})$ for every digit $d$.
Therefore, we can model this probabilistic aspect in the logic by considering the predicate interpretation as probabilistic.
We can then define logic PL problems by following this idea.

\paragraph{Logic PL problem.}
Consider a logic $\Logic$, a domain $\Delta$ and a closed (i.e without variables) logical formula $\phi$ over $\Logic$.
We define:
\begin{itemize}
    \item $\Gamma_{\Logic, \Delta} = (\Omega_{\Logic, \Delta}, \tribe_{\Logic, \Delta}, \prob_{\Logic, \Delta})$ is the probabilistic space of $\Logic$ in the real context $\Delta$.
    It models the probabilistic aspect of facts by taking $\Omega_{\Logic, \Delta}$ as the set of all predicate interpretations $[.]_\Pred$ for $\Delta$ that are possible.
    $\prob_{\Logic, \Delta}$ is specific to the real context and models the likelihood of each $[.]_\Pred$.
    \item $\inputspace_{\Logic, \Delta}$ is a subset of all possible function interpretations $[.]_\Func$ for $\Delta$.
    It models the possible observations about the reality and depends on the real context.
    \item $B_{\Logic, \Delta}^\phi$ is the function that maps interpretations of $\inputspace_{\Logic, \Delta}$ to binary random variables in $\Gamma_{\Logic, \Delta}$ corresponding to $\phi$:
    \begin{center}
        $B_{\Logic, \Delta}^\phi([.]_\Func)([.]_\Pred) = [\phi]_{(\Delta,[.]_\Func, [.]_\Pred)}$.
    \end{center}
    \item $P_{\Logic, \Delta}^\phi = (\Gamma_{\Logic, \Delta}, \inputspace_{\Logic,\Delta}, \{0,1\}, B_{\Logic, \Delta}^\phi)$ is the logic PL problem of $\phi$.
\end{itemize}

By taking inspiration from ProbLog and DeepProbLog, we now define a class of PNL systems for logic PL problems.
Consider a logical formula $\phi$.
The idea is to decompose $P_{\Logic, \Delta}^\phi$ into $m$ smaller logic PL problems $P_{\Logic, \Delta}^{\phi_1},...,P_{\Logic, \Delta}^{\phi_m}$ where $\phi_1,...,\phi_m$ represent simpler concepts.
A symbolic function $S$, based on axioms about the context, combines $\phi_1,...,\phi_m$ to determine the truth value of $\phi$.

\paragraph{Logic PNL system.}
Consider a logic $\Logic$ and a domain $\Delta$.
Suppose the existence of an axiom set $A = \{\alpha_1, ..., \alpha_n\}$ such that all $\alpha_i$ are true in $\Gamma_{\Logic, \Delta}$:
\begin{equation*}
    \bigcup_{[.]_\Func \in \inputspace_{\Logic, \Delta}}(B_{\Logic,\Delta}^{\alpha_i}([.]_\Func) = 0) = \emptyset
\end{equation*}
We write $\{\Phi_1,..., \Phi_l\} \vdash \Phi$ to denote that there exists a proof of $\Phi$ in $\Logic$ assuming $\Phi_1,...,\Phi_l$.
Now suppose there exist formulas $\phi_1,...,\phi_m$ such that:
\begin{itemize}
    \item $(P_{\Logic, \Delta}^{\phi_k})_{1\leq k \leq m}$ are independent logic PL problems,
    \item for all $\phi'_1,...,\phi'_m$ such that $\phi'_k = \phi_k$ or $\phi'_k = \neg \phi_k$:
    \begin{itemize}
        \item $A \cup \{ \phi'_1,...,\phi'_m \} \vdash \phi$ or $A \cup \{ \phi'_1,...,\phi'_m \} \vdash \neg\phi$
        \item $A \cup \{ \phi'_1,...,\phi'_m \} \not \vdash \bot$.
    \end{itemize}
    We say that the $\phi_1,...,\phi_m$ decide $\phi$ and are coherent with respect to $A$.
\end{itemize}
Under these assumptions, we have sufficient elements to prove that the following constitutes a PNL system for $P_{\Logic, \Delta}^\phi$: 
$\bigl((\{0, 1\})_{1\leq k \leq m}, (B_{\Logic, \Delta}^{\phi_k})_{1\leq k \leq m}, (\hat{\distrib}_k)_{1\leq k \leq m}, S\bigr)$, where:
\begin{itemize}
    \item for all $k \in [1,m]$, $\hat{\distrib}_k$ is a neural network which takes as input ground terms of $\phi_k$ interpreted as elements of $\Delta$ and returns an estimation of the probability distribution of $B_{\Logic, \Delta}^{\phi_k}$,
    \item for all $x_1,...,x_m \in \{0,1\}$, $S(x_1,...,x_m)$ is equal to \begin{center}
        $\begin{cases}
            1 \text{ else if } A \cup \{\phi_k | x_k=1\} \cup \{\neg \phi_k | x_k = 0\} \vdash \phi, \\
            0 \text{ otherwise.}
        \end{cases}$
    \end{center}
\end{itemize}
The proof is given in Appendix~\ref{proof:LogicPNLSystemCorrectness}.

\paragraph{Oracles for Logic PNL systems.}
Consider a logic PNL system such that the theory ($\Logic$,$A$) is decidable.
Then, Algorithm~\ref{alg:LogicOracle} describes the pseudocode of a valid and complete oracle for the function $S$ based on a decision procedure of ($\Logic$,$A$).
The proof is given in Appendix~\ref{proof:LogicOracleValidityCompletness}.

\begin{algorithm}[tb]
    \caption{Logic Oracle}
    \label{alg:LogicOracle}
    \textbf{Input}: $v \in \setofvaluations{}$, $o \in \outputspace$ \\
    \textbf{Output}: A valid oracle output for $S$
    \begin{algorithmic}[1] 

        \IF {$A \cup \{\phi_k | v[k] = 1\} \cup \{\neg \phi_k | v[k] = 0\} \vdash \neg \phi$}
            \STATE $res \gets 0$
        \ELSIF {$A \cup \{\phi_k | v[k] = 1\} \cup \{\neg \phi_k | v[k] = 0\} \vdash \phi$}
            \STATE $res \gets 1$.
        \ELSE
            \STATE $res \gets \unknown$.
        \ENDIF

        \IF {$res \neq \unknown \land o = 0$}
            \STATE $res \gets 1-res$.
        \ENDIF

        \STATE \textbf{return} $res$.
        
    \end{algorithmic}
\end{algorithm}

We can link this to standard frameworks such as DeepProbLog.
Indeed, the logic $\Logic$ corresponds to the logic of Prolog.
$\Delta$ is the set of possible inputs for neural networks, which are real-world data such as images.
The formula $\phi$ represents the query.
The rules annotated by neural neutworks in DeepProbLog correspond to the $\phi_1,...,\phi_m$, and the neural networks producing their probabilities are $\hat{\distrib}_1,...,\hat{\distrib}_k$.
Finally, $S$ corresponds to the logical provenance formula, as $S^{-1}(1)$ represents the set of proofs of $\phi$.
However, DeepProbLog differs slightly because it does not consider negation.
A detailed application to DeepProbLog is provided in Appendix~\ref{appendix:applicationtodeepproblog}.
This application is particularly interesting when the DeepProbLog program does not contain functors.
Indeed, in a Prolog program without negation and functors the query answering task is polynomial time in the size of the program as it can be transformed into an equivalent Horn-SAT instance.
An oracle based on a solver can thus answer in polynomial time with respect to the size of the program.
In comparison, under these conditions, the size of the logical provenance computed by DeepProbLog could be factorial in the square root of the size of the program (see Example~\ref{ex:GraphReach}).

\section{ApproxDPNL}
\begin{algorithm}[tb]
    \caption{ApproxDPNL($i,o,O_S,$\textsc{stop},\textsc{h})}
    \label{alg:ApproxDPNL}
    \textbf{Input}: $i \in \inputspace$, $o \in \outputspace$, $v \in \setofvaluations{}$, $O_S$ an oracle for $S$ \\
    \textsc{stop} : function taking a lower bound and an upper bound of the result and tells the algorithm to stop or not. \\
    \textsc{h} : An heuristic that choose a valuation from a set. \\
    \textbf{Output}: An estimation of $\prob(F(i) = o | w(v))$ \\
    \begin{algorithmic}[1] 
        \STATE Set $low, up := 0, 1$.
        \STATE Set $queue := \{[\unknown,...,\unknown]\}$.
        \WHILE {$|queue| > 0$ and $\neg \textsc{stop}(low,up)$}
            \STATE Set $v = \textsc{h}(queue)$.
            \STATE $queue \gets queue \backslash \{v\}$
            \STATE Set $K := \{k \in [1,m] | v[k] \neq \unknown\}$.
                \STATE Set $prob := \Pi_{k\in K} \prob(X_k(i) = v[k])$.
            \IF {$O_S(v,o) = 1$}
                \STATE $low \gets low + prob$
            \ELSIF {$O_S(v,o) = 0$}
                \STATE $up \gets up - prob$
            \ELSE
                \STATE Choose $k \in [1,m]$ such that $v[k] = \unknown$.
                \STATE $queue \gets queue \cup \{v_{v[k]\gets y} | y \in \Vset{k}\}$
            \ENDIF
        \ENDWHILE
        \STATE \textbf{return} $\sqrt{low \times up}$
    \end{algorithmic}
\end{algorithm}
DPNL can be readily adapted to approximate reasoning with guarantees on the approximation error. 
We consider $(\epsilon, \delta)$-approximation defined as follows, where $\epsilon$ denotes the relative error bounded by a parameter $\epsilon$ with probability $1-\delta$ \cite{chakraborty,WMC:Approx:Dubray2024}.

\paragraph{$(\epsilon, \delta)$-approximation.} An estimator $\hat{y}$ is an $(\epsilon, \delta)$-approximation when $\prob(|\frac{y-\hat{y}}{y}| > \epsilon) \leq \delta$.

\paragraph{$\epsilon$-approximation.} An estimator $\hat{y}$ is an $\epsilon$-approximation when $y / (1+\epsilon) \leq \hat{y} \leq y \cdot (1+\epsilon)$.

DPNL with $(\epsilon, \delta)$-approximations, named ApproxDPNL, is proposed in Algorithm \ref{alg:ApproxDPNL}. 
As proposed in \cite{WMC:Approx:Dubray2024}, ApproxDPNL can be stopped at any time $T$ and provides lower and upper bound guarantees on the result of the exact weighted model counting. 
The choice of the function \textsc{stop} determines the guarantee of the result, where $low$ and $up$ are the lower and upper bound of the approximation interval. 
 \begin{itemize}
    \item If \textsc{stop}($low,up$)$:= \begin{cases}
        1 \text{ if } up \leq low \cdot (1+\epsilon)^2 \\
        0 \text{ else }
    \end{cases}$, the result of ApproxDPNL is an $\epsilon$-approximation of $\prob(F(i) = o | w(v))$ (see \cite{WMC:Approx:Dubray2024}).
    \item If \textsc{stop}($low,up$)$:= \begin{cases}
        1 \text{ if } up-low \leq \epsilon \\
        0 \text{ else }
    \end{cases}$, the result of ApproxDPNL $r$ respects the following property $\prob(F(i) = o | w(v) - \epsilon \leq r \leq \prob(F(i) = o | w(v) + \epsilon$.
    \item If \textsc{stop} returns $1$ after time $T$ has passed, the result of ApproxDPNL is a biased approximation of $\prob(F(i) = o | w(v))$ that took time $T$ to compute.
\end{itemize}
The heuristic \textsc{h} controls in which order ApproxDPNL explores sub-valuations.
Intuitively, it is important to prioritise the exploration of sub-valuations that are the most probable so that the size of the approximation interval $|up-low|$ narrows faster.
However, it is also possible for use a random choice as \textsc{h}, which results in uniform sampling. 
 
\section{Experiments}
We experimentally evaluate DPNL in order to answer the following research questions:

\begin{itemize}
    \item[\textbf{Q\#1}] How does DPNL compare to PNL systems with exact reasoning in terms of efficiency and accuracy? 
    \item[\textbf{Q\#2}] How does DPNL compare to PNL systems with approximate reasoning in the state-of-the-art in terms of efficiency and accuracy? 
    \item[] \item[\textbf{Q\#3}] How does ApproxDPNL compare to exact and approximate reasoning methods in the state-of-the-art in terms of efficiency and accuracy? 
\end{itemize}
\paragraph{Task.}
We evaluate DPNL on the popular MNIST-N-SUM task \cite{NeSy:DeepProbLog:Manhaeve2018} as introduced in Section~\ref{ex:mnist_sum}. 
By extending from an N digits sum to an N+1 digits sum, we increase the complexity of the reasoning task exponentially. 
Thus, the multi-digit MNIST-N-SUM task is well-suited to investigate the scaling properties of DPNL.
We use the MNIST dataset from \cite{Article:LeCun:MNIST} that contains 60K training images and 10K test images. 
In consequence, with respect to N, the training set contains 60K / 2N images and the test set contains 10K / 2N images so that every image is used once.  
Each training sample consists of the images for the two summands and the supervision label on their ground truth sum. 

\paragraph{Experimental Setup.}
We compare DPNL with the PNL systems DeepProbLog \cite{NeSy:DeepProbLog:Manhaeve2018} and Scallop \cite{NeSy:Scallop:Huang2021} that use exact inference (\textbf{Q\#1}), and with A-Nesi \cite{NeSy:A-NeSI:vanKrieken2023}, DPLA$^*$ \cite{NeSy:DeepProbLog:Manhaeve2021} and Scallop (with the top-k semiring), which use approximate inference (\textbf{Q\#2}).
A-Nesi is evaluated across its different variants (predict only, explain, prune). 
We further test ApproxDPNL compare the results to the previously mentioned exact and approximate methods (\textbf{Q\#3}).
We choose a range of different values for reasoning timeouts $T=\{0.05, 0.075, 0.1, 0.5\}$ for ApproxDPNL. 
We evaluate the test accuracy and the runtime of the symbolic reasoning component.
To ensure comparability, we maintain consistent experimental settings across all models, by reproducing baseline results on the same hardware and using the same parameters.

Furthermore, we also want to compare exact and approximate reasoning. 
To address this, we use a common set of hyperparameters. 
For all methods, we employ the MNIST classifier architecture proposed in \cite{NeSy:DeepProbLog:Manhaeve2018}.
The Adam optimizer \cite{kingma2017adam}, a batch size of 2, a single training epoch, and a learning rate of $1\mathrm{e}{-3}$ are used.
Each experiment is performed with multiple independent runs, and we report the average test accuracy and the average reasoning time per sample during training.
Further, we enforce timeout thresholds, terminating any experiment with a run that exceeds the limit of six hours.
Details on hyperparameter values, hardware and information on baseline configurations 
are provided in Appendix~\ref{appendix:experimental_details}.

\paragraph{Results}
\begin{table*}[ht]
    \centering
    \renewcommand{\arraystretch}{1.2} 
    \resizebox{\textwidth}{!}{ 
    \begin{tabular}{|l||c|c|c||c|c|c||c|c|c||c|c|c|}
    \hline
    \textbf{} & \multicolumn{3}{c||}{\textbf{$\mathbf{N=1}$}} & \multicolumn{3}{c||}{\textbf{$\mathbf{N=2}$}} & \multicolumn{3}{c||}{\textbf{$\mathbf{N=3}$}} & \multicolumn{3}{c|}{\textbf{$\mathbf{N=4}$}} \\
    \hline
                    & Acc & Std & Time & Acc & Std & Time & Acc & Std & Time & Acc & Std & Time \\
    \hline
    \textbf{DPNL}     & 0.9550 & 0.0090 & 0.0011 & 0.9116 & 0.0106 & \textbf{0.0099} & 0.8787 & 0.0100 & 0.1079 & 0.8262 & 0.0211 & 1.2485 \\
    \textbf{DeepProbLog}  & 0.9556 & 0.0060 & 0.0238 & 0.9098 & 0.0122 & 0.4784 & T/O & T/O &  T/O & T/O & T/O& T/O \\
    \textbf{Scallop (exact)}  & 0.9515 & 0.0037 & 0.0594 & T/O & T/O & T/O & T/O & T/O & T/O & T/O & T/O & T/O \\
    \hline
    \hline 
    \textbf{ApproxDPNL (T=0.5)}  & 0.9552 & 0.0100 & 0.0038 & \textbf{0.9147} & 0.0165 & 0.0163 & 0.8842 & 0.0079 & 0.1694 & 0.8334 & 0.0231 & 0.5381 \\
    \textbf{ApproxDPNL (T=0.1)}  & 0.9537 & 0.0107 & 0.0037 & 0.9121 & 0.0230 & 0.0166 & \textbf{0.8892} & 0.0108 & 0.0962 & 0.8293 & 0.0192 & 0.1052 \\
    \textbf{ApproxDPNL (T=0.075)}  & 0.9524 & 0.0097 & 0.0038 & 0.9055 & 0.0251 & 0.0160 & 0.8709 & 0.0146 & 0.0747 & \textbf{0.8404} & 0.0080 & 0.0797 \\
    \textbf{ApproxDPNL (T=0.05)}  & \textbf{0.9570} & 0.0086 & 0.0034 & 0.9060 & 0.0184 & 0.0162 & 0.8616 & 0.0126 & 0.0512 & 0.8354 & 0.0165 & 0.0535 \\
    \textbf{DPL A*}  & 0.8120 & 0.1635 & 0.0213 & 0.8945 & 0.0146 & 0.0255 & 0.8365 & 0.0303 & 0.1237 & T/O & T/O & T/O \\
    \textbf{Scallop (top-k)}  & 0.9512 & 0.0076 & \textbf{0.0009} & 0.8257 & 0.2060 & 0.0244 & 0.7523 & 0.2612 & 0.9521 & T/O & T/O & T/O \\
    \textbf{A-Nesi (predict)}  & 0.9435 & 0.0071 & 0.0328 & 0.8849 & 0.0142 & 0.0265 & 0.7796 & 0.0074 & \textbf{0.0170} & 0.5929 & 0.0394 & \textbf{0.0221} \\
    \textbf{A-Nesi (explain)}  & 0.9510 & 0.0066 & 0.0348 & 0.8952 & 0.0071 & 0.0213 & 0.7707 & 0.0179 & 0.0219 & 0.6584 & 0.0239 & 0.0268 \\
    \textbf{A-Nesi (prune)}  & 0.9489 & 0.0070 & 0.0251 & 0.7866 & 0.2225 & 0.0195 & 0.7706 & 0.0214 & 0.0206 & 0.5905 & 0.0339 & 0.0245 \\
    \hline
    \end{tabular}
    }
    \caption{Results on the MNIST-$N$-SUM task for $N=\{1,2,3,4\}$.
    We report the average test accuracy (acc), its standard deviation (std) and the average reasoning time per sample (time) in seconds.
    T/O denotes a computational timeout. 
    The systems with exact reasoning are DPNL, DeepProbLog and Scallop. 
    The systems with approximate reasoning are DPLA$^*$, Scallop with the top-k-proofs semiring and A-Nesi in the variants predict-only (predict), explain and prune.
    The highest test accuracy numbers and the shortest reasoning times are marked in bold.}
    \label{tab:performance_comparison}
\end{table*}
The results of the experiments are given in Table~\ref{tab:performance_comparison}. 
For \textbf{Q\#1}, we first compare DPNL with other PNL systems using exact inference. 
DeepProbLog achieves similar accuracy to DPNL for N=1 and N=2. 
In terms of reasoning time per sample, DPNL outperforms DeepProbLog, with the advantage becoming more noticeable as task complexity increases.
Scallop with exact reasoning only succeeds to complete the task $N=1$. 
For $N\geq3$, DeepProbLog fails to finish, while DPNL achieves an accuracy of 0.82 for $N=4$ with an average reasoning time per sample of $1.24$s. 
Among all the exact PNL systems considered, DPNL is the only one that scales to $N=4$. 

When comparing exact and approximate reasoning systems for \textbf{Q\#2}, exact reasoning positively impacts accuracy, and this effect becomes more pronounced as the task difficulty increases.
For $N=1$, the mean test accuracy of all systems considered is centered around the value $0.95$ (except DPLA$^*$).
However, for $N=2$ there is already a noticeable accuracy gap between approximate methods and exact methods such as DeepProbLog and DPNL.
This gap becomes even larger for $N=3$, where DPNL achieves a considerably higher accuracy of $0.87$ compared to the approximate methods, which do not exceed $0.83$.
For $N=4$, even the approximate methods such as DPLA$^*$ and Scallop (top-k) fail to complete their runs, while DPNL finishes successfully.
The only baseline that finishes for $N=4$ is A-Nesi. 
All variants of A-Nesi have noticeably lower reasoning times per sample compared to DPNL. 
However, DPNL has a clear advantage in terms of accuracy. 
The most accurate A-Nesi model is the explain variant, which achieves an accuracy of $0.65$, while DPNL reaches an accuracy of $0.82$.

Regarding \textbf{Q\#3}, ApproxDPNL controls the total runtime by setting a timeout $T$ during the reasoning process per sample. 
Over all tasks $N=\{1,2,3,4\}$, ApproxDPNL achieves an accuracy competitive to DPNL with exact reasoning. 
For the approximate reasoning methods, although A-Nesi still has an advantage in terms of run time per sample, ApproxDPNL clearly outperforms A-Nesi in terms of test set accuracy. 

\section{Related Works}
There is a large body of work on neuro-symbolic methods aimed at bridging the gap between learning and reasoning. Within this field, PNL systems can be categorized into two groups: those based on exact reasoning and those based on approximate reasoning.

\paragraph{PNL systems with exact reasoning.}
DeepProbLog \cite{NeSy:DeepProbLog:Manhaeve2018}, one of the first PNL systems, is often used as a benchmark for accuracy, establishing itself as a foundational framework in PNL.
DeepProbLog is based on a probabilistic extension of Prolog called ProbLog \cite{NeSy:ProbLog:DeRaedt2007} where neural networks are used to estimate probabilities of extensional predicates. 
ProbLog reduces the probabilistic reasoning task to the PWMC problem.
However, two key scalability issues arise. 
First, the logical provenance is initially computed in DNF, whose size can grow exponentially depending on the task. Even when the clauses of the logical provenance are stored in a prefix tree, its size can grow exponentially with the number of neural ground atoms.
Second, solving the PWMC itself is \#P-hard causing blowups in the size of the compiled forms of the logical provenance. Specifically, the size of the Binary Decision Diagram  can increase exponentially (as detailed in Appendix~\ref{appendix:applicationtodeepproblog}). 

NeurASP \cite{NeurASP} uses a similar approach but relies on Answer Set Programming. 
It considers the neural network output as the probability distribution over atomic
facts in answer set programs.

As noticed in \cite{NeSy:A-NeSI:vanKrieken2023}, DeepProbLog, NeurASP and other methods in this line of works \cite{pmlr-v80-xu18h,ahmed2022} provide exact inference with probabilistic circuit methods \cite{ProbCirc20}. 
However, they face significant scaling issues and struggle to handle large problem instances effectively. 
In comparison, DPNL provides an alternative approach, that does not depend on a reduction that requires \emph{binary} random variables. Instead, it makes it possible to write oracles with respect to any symbolic function taking general values as input. 

\paragraph{PNL systems with approximate reasoning.}

To address the scalability issue, several approaches have been developed to relax the exact reasoning process with approximations. 

DeepStochLog \cite{NeSy:DeepStochLog:Winters2022} employs the semantics of stochastic definite clause grammars that achieves better scaling properties thanks to a reduced search space. 

Scallop \cite{NeSy:Scallop:Huang2021} computes only the $k$ most probable clauses of the logical provenance  and performs PWMC on the disjunction of these $k$ clauses. 
This approach thus limits both the size of the formula and the complexity of PWMC to a polynomial function of $k$. 
DPLA$^*$ \cite{NeSy:DeepProbLog:Manhaeve2021} is based on DeepProbLog, but uses only a subset of all proofs. 
It employs A*-search and heuristics to efficiently search for the best proofs. 
However, while being polynomial in complexity, these approximations lack guarantees and may fail when the logical provenance's essential structure is not captured by a few clauses.

Other systems relying on similar polynomial approximations have been developed. In particular, A-Nesi \cite{NeSy:A-NeSI:vanKrieken2023} uses deep generative modeling to improve scalability.
It employs a neural prediction model that performs approximate inference over PWMC problems and a neural explanation model that computes which world best explains a prediction. 
A prior belief function is used to sample data corresponding to the knowledge in order to train these models. 
A-Nesi can also be used in conjunction with a symbolic pruner that sets the probability of certain variables to zero and reduces the search space.

One of the important drawbacks of using approximations is that it undermines one of the primary strengths of PNL: confidence in the model's reliability.
Without a guarantee that the approximation is close to the true result, the model's reliability may be compromised. 
%
%
Therefore, it makes sense to employ approximations with guarantees, especially during inference when the model is deployed.

In comparison, DPNL makes it possible to obtain more accurate models since it pushes the boundaries of exact inference further.
In addition, the extension to ApproxDPNL enables even better scalability through approximations, while providing guarantees on the approximation error.

\section{Conclusion}
We have presented DPNL, a new approach for exact probabilistic reasoning in neurosymbolic learning. The core feature of DPNL lies in its use of oracles, which not only eliminate the need for computing logical provenance formulas, but also allow for significant acceleration of the inference process. 
Oracles can be automatically generated or customized for further efficiency gains, making DPNL a flexible solution.
The correctness and termination of DPNL are formally proved.
We furthermore introduced ApproxDPNL that provides approximate reasoning with guarantees on the approximation error in the context of DPNL. 
Experimental results firstly show that DPNL outperforms existing exact inference methods, handling more complex tasks with greater efficiency and producing more accurate models compared to approximate reasoning techniques.
Secondly, ApproxDPNL has shown to further improve the scalability of neurosymbolic programming by incorporating approximations with guarantees on the tolerated error at inference.
DPNL and ApproxDPNL thus lay a promising foundation for the development of highly accurate models across a wide range of challenging tasks that involve perception and reasoning.

\bibliographystyle{named}
\bibliography{ijcai25}

\newpage
\appendix

\section{ProbDPLL}\label{appendix:proof:ProbDPLLCorrectness}

\subsection{Notation details}

ProbDPLL is described in Algorithm \ref{alg:ProbDPLL}.
We first refine some notations.
$\mathbf{X}$ is a set of (independent) binary (random) variables.
A Boolean formula is a function from the set of valuations $\Functions{\mathbf{X}}{\{0,1\}}$ into $\{0, 1\}$.
For all $X \in \mathbf{X}$, we also denote by $X$ the Boolean formula such that for all valuation $v \in \Functions{\mathbf{X}}{\{0,1\}}$, $X(v) = 1$ iff $v(X)=1$.
We use logical operators $\land$, $\lor$ and $\neg$ to build Boolean formulas.
A Boolean formula $G$ is in CNF when
\begin{center}
    $G = \overset{N}{\underset{i=1}{\land}}\overset{M_i}{\underset{j=1}{\lor}}L_{i,j}$
\end{center}
where $N$ is the number of clauses in $G$, $M_i$ is the number of atoms in the $i$-th clause of $G$ and $L_{i,j}$ is a literal $X$ or $\neg X$ with $X\in \mathbf{X}$ for all $i,j$.
For all $X \in \mathbf{X}$,
\begin{itemize}
    \item $G_{|X=0}$ denotes $G$, where every clause containing $\neg X$ is set to true and where $X$ is eliminated from every clause in which it occurred.
    \item $G_{|X=1}$ denotes $G$, where every clause containing $X$ is set to true and where $\neg X$ is eliminated from every clause in which it occurred.
\end{itemize}
Consider the following relation $G \equiv (G_{|X=0} \land \neg X) \lor (G_{|X=1} \land X)$ (see \cite{WMC:DPLL:Davis1962}).
Given a probability distribution $\sigma : \mathbf{X} \rightarrow [0,1]$ of the variables in $\mathbf{X}$, we note $\text{PWMC}_G^\sigma$ the probabilistic weighted model count of $G$ w.r.t. $\sigma$ as defined in Section \ref{definition:PWMC}.

\subsection{Vizualisation} \label{appendix:probdpllviz}
Figure~\ref{fig:ProbDPLL} illustrates the trace of $\text{ProbDPLL}(G, \sigma)$, see Algorithm \ref{alg:ProbDPLL}, for the example formula $G = (A \lor B \lor \neg C) \land (A \lor B \lor C) \land (B \lor \neg A) \land (\neg B \lor C) \land (\neg B \lor \neg C)$ as a tree. 
Each node stands for a recursive call.
Since $G$ is not satisfiable, $PWMC_G^\sigma$ returns $0$.

\begin{figure}[H]

    \begin{tikzpicture}[every node/.style={fill=white}]
        \tiny
        
        \node[rectangle, draw] (f) at (0,0) {
            \begin{tabular}{l}
                ProbDPLL \\
                \hline
                Input : $G = (A \lor B \lor \neg C) \land (A \lor B \lor C) \land $\\$ (B \lor \neg A) \land (\neg B \lor C) \land (\neg B \lor \neg C)$ \\
                Input : $\sigma$ \\
                \hline
                Return : $0$
            \end{tabular}           
        };

        \node[rectangle, draw] (f0) at (-0.25,-2) {
            \begin{tabular}{l}
                ProbDPLL \\
                \hline
                Input : $G = (A \lor \neg C) \land (A \lor C) \land \neg A$ \\
                Input : $\sigma$ \\
                \hline
                Return : $0$
            \end{tabular}           
        };
        \path [->, thick] (f) edge node {$G_{|B=0}$} (f0);

        \node[rectangle, draw] (f00) at (-1.25,-4) {
            \begin{tabular}{l}
                ProbDPLL \\
                \hline
                Input : $G=\neg C \land C$ \\
                Input : $\sigma$ \\
                \hline
                Return : $0$
            \end{tabular}           
        };
        \path [->, thick] (f0) edge node {$G_{|A=0}$} (f00);
    
        \node[rectangle, draw] (f000) at (-1.6,-6) {
            \begin{tabular}{l}
                ProbDPLL \\
                \hline
                Input : $G=\bot$ \\
                Input : $\sigma$ \\
                \hline
                Return : $0$
            \end{tabular}           
        };
        \path [->, thick] (f00) edge node {$G_{|C=0}$} (f000);
        
        \node[rectangle, draw] (f001) at (0.5,-6) {
            \begin{tabular}{l}
                ProbDPLL \\
                \hline
                Input : $G=\bot$ \\
                Input : $\sigma$ \\
                \hline
                Return : $0$
            \end{tabular}           
        };
        \path [->, thick] (f00) edge node {$G_{|C=1}$} (f001);

        \node[rectangle, draw] (f01) at (3,-6) {
            \begin{tabular}{l}
                ProbDPLL \\
                \hline
                Input : $G=\bot$ \\
                Input : $\sigma$ \\
                \hline
                Return : $0$
            \end{tabular}           
        };
        \path [->, thick] (f0) edge node {$G_{|A=1}$} (f01);

        \node[rectangle, draw] (f1) at (4.5,-2) {
            \begin{tabular}{l}
                ProbDPLL \\
                \hline
                Input : $G = C \land \neg C$ \\
                Input : $\sigma$ \\
                \hline
                Return : $0$
            \end{tabular}           
        };
        \path [->, thick] (f) edge node {$G_{|B=1}$} (f1);
    
        \node[rectangle, draw] (f10) at (5.5,-4) {
            \begin{tabular}{l}
                ProbDPLL \\
                \hline
                Input : $G=\bot$ \\
                Input : $\sigma$ \\
                \hline
                Return : $0$
            \end{tabular}          
        };
        \path [->, thick] (f1) edge node {$G_{|C=1}$} (f10);
        
        \node[rectangle, draw] (f11) at (3.25,-4) {
            \begin{tabular}{l}
                ProbDPLL \\
                \hline
                Input : $G=\bot$ \\
                Input : $\sigma$ \\
                \hline
                Return : $0$
            \end{tabular}              
        };
        \path [->, thick] (f1) edge node {$G_{|C=0}$} (f11);
                
    \end{tikzpicture}

    \caption{Trace of $\text{ProbDPLL}(G,\sigma)$ where $G = (A \lor B \lor \neg C) \land (A \lor B \lor C) \land (B \lor \neg A) \land (\neg B \lor C) \land (\neg B \lor \neg C)$.}
    \label{fig:ProbDPLL}
\end{figure}

\subsection{Termination and Correctness}
In this Section we proof the termination and correctness of ProbDPLL. 

\begin{theorem}
    For every 
    \begin{itemize}
        \item Boolean formula $G = \overset{N}{\underset{i=1}{\land}}\overset{M_i}{\underset{j=1}{\lor}}L_{i,j}$ in CNF over a finite set of (independent) binary (random) variables $\mathbf{X}$,
        \item probability distribution $\sigma : \mathbf{X} \rightarrow [0,1]$,
    \end{itemize} 
    $\text{ProbDPLL}(G,\sigma)$ terminates and returns $\text{PWMC}_G^\sigma$.
\end{theorem}
\begin{proof}
    We proceed by induction for $n$ over the number of variables $|\mathbf{X}|$ in $G$. 
    We define the induction hypothesis $H_n$:
    \begin{quote}
    \emph{For every Boolean formula $G = \overset{N}{\underset{i=1}{\land}}\overset{M_i}{\underset{j=1}{\lor}}L_{i,j}$ in CNF over a finite set of (independent) binary (random) variables $\mathbf{X}$ and a probability distribution $\sigma : \mathbf{X} \rightarrow [0,1]$ such that $|\mathbf{X}| \leq n$, $\text{ProbDPLL}(G,\sigma)$ terminates and returns $\text{PWMC}_G^\sigma$.}
    \end{quote}

    \noindent We can then prove $H_n$ for all $n \geq 0$ by induction:
    \paragraph{\underline{\normalfont{Base case }$H_0$}:} Let
        \begin{itemize}
            \item $G = \overset{N}{\underset{i=1}{\land}}\overset{M_i}{\underset{j=1}{\lor}}L_{i,j}$ a Boolean formula in CNF over $\mathbf{X}$ a finite set of (independent) binary (random) variables,
            \item $\sigma : \mathbf{X} \rightarrow [0,1]$,
        \end{itemize}
        such that $|\mathbf{X}| \leq 0$.
        Therefore, $G$ is constant because $\Functions{\emptyset}{\{0,1\}}$ only contains one element $v_0$.

        Moreover,
        \begin{align*}
            \text{PWMC}_G^\sigma := & \sum_{v \in G^{-1}(1)}\prod_{X \in v^{-1}(1)}\sigma(X)\prod_{X \in v^{-1}(0)}(1-\sigma(X)) \\
                                  = & \sum_{v \in G^{-1}(1)}1 \\
                                  = & \text{ }G(v_0).
        \end{align*}
        There are two cases:
        \begin{itemize}
            \item $N=0$:
            In this case $\text{ProbDPLL}(G,\sigma)$ terminates and returns $1$ because $G$ has no clauses.
            Since $N=0$, $G$ is constant and equal to the empty conjunction, so $G$ is equal to $1$.
            It follows $\text{ProbDPLL}(G,\sigma) = \text{PWMC}_G^\sigma = 1$.
            \item $N>0$:
            In this case, $\text{ProbDPLL}(G,\sigma)$ terminates and returns $0$ because $G$ contains the empty clause.
            Since $N>0$, $G$ is constant and equal to a conjunction of empty clauses, so $G$ is constant and equal to $0$.
            It follows $\text{ProbDPLL}(G,\sigma) = \text{PWMC}_G^\sigma = 0$.
        \end{itemize}

        \paragraph{\underline{\normalfont{Induction step: }$H_n \implies H_{n+1}$}:} Let $n\geq0$ and suppose that $H_n$ holds. Consider
        \begin{itemize}
            \item a Boolean formula $G = \overset{N}{\underset{i=1}{\land}}\overset{M_i}{\underset{j=1}{\lor}}L_{i,j}$ in CNF over a finite set of (independent) binary (random) variables $\mathbf{X}$,
            \item a probability distribution $\sigma : \mathbf{X} \rightarrow [0,1]$
        \end{itemize}
        such that $|\mathbf{X}| \leq n+1$.
        If $|\mathbf{X}| < n+1$, we can directly apply $H_n$ to conclude. 
        So we consider that $|\mathbf{X}| = n+1$.
        Let $X\in \mathbf{X}$ be the chosen variable in line 6 of Algorithm \ref{alg:ProbDPLL}. 
        We have:
        \begin{itemize}
            \item $G \equiv G_0 \lor G_1$ with $G_0 = (G_{|X=0} \land \neg X)$ and $G_1 = (G_{|X=1} \land X)$.
            \item Since $G_0$ contains $\neg X$, it follows for all $v \in G_0^{-1}(1)$ $v(X) = 0$.
            \item Since $G_1$ contains $X$, it follows for all $v \in G_1^{-1}(1)$, $v(X) = 1$.
        \end{itemize}
        Therefore, $G_0^{-1}(1)$ and $G_1^{-1}(1)$ are disjoint and form a partition of $G^{-1}(1)$.
        Since $\text{PWMC}_G^\sigma := $
        \begin{align*}
        & \sum_{v \in G^{-1}(1)}\prod_{Y \in v^{-1}(1)}\sigma(Y) \cdot \prod_{Y \in v^{-1}(0)}(1-\sigma(Y))
        \end{align*}
        by dividing the sum in two parts, we obtain $\text{PWMC}_G^\sigma = \text{PWMC}_{G_0}^\sigma + \text{PWMC}_{G_1}^\sigma$.

        We now prove that
        \begin{itemize}
            \item $\text{PWMC}_{G_0}^\sigma = (1-\sigma(X))\cdot\text{PWMC}_{G_{|X=0}}^\sigma$ and 
            \item $\text{PWMC}_{G_1}^\sigma = \sigma(X)\cdot\text{PWMC}_{G_{|X=1}}^\sigma$.
        \end{itemize}
        Since the two proofs are similar, we only show the proof for $G_1$.
        For every $v \in G_1^{-1}(1)$, the restriction of $v$ to $\mathbf{X}\backslash \{X\}$, noted as $v_{|\mathbf{X}\backslash \{X\}}$, is a model of $G_{|X=1}$, i.e. $v_{|\mathbf{X}\backslash X} \in G_{|X=1}^{-1}(1)$.
        And for every $v' \in G_{|X=1}^{-1}(1)$ the extension of $v'$ to $\mathbf{X}$ by setting $v'(X)=1$ is a model of $G_1$.
        Indeed, $G_{|X=1}$ is a Boolean formula over $\mathbf{X}\backslash \{X\}$ and $G_1 = (G_{|X=1} \land X)$.
        From this follows: 
        \begin{align*}
            \text{PWMC}_{G_1}^\sigma := & \sum_{v \in G_1^{-1}(1)}\prod_{Y \in v^{-1}(1)}\sigma(Y)\prod_{Y \in v^{-1}(0)}(1-\sigma(Y)) \\
            = & \sum_{v' \in G_{|X=1}^{-1}(1)}(\sigma(X)\cdot\prod_{Y \in v'^{-1}(1)}\sigma(Y) \cdot\\
              & \prod_{Y \in v'^{-1}(0)}(1-\sigma(Y))) \\
            = & \sigma(X) \cdot \text{PWMC}_{G_{|X=1}}^\sigma.
        \end{align*}
        Thus $\text{PWMC}_{G_1}^\sigma = \sigma(X)\cdot\text{PWMC}_{G_{|X=1}}^\sigma$.
        Following the same considerations, we derive $\text{PWMC}_{G_0}^\sigma = (1-\sigma(X))\cdot\text{PWMC}_{G_{|X=0}}^\sigma$.

        Since $G_{|X=1}$ and $G_{|X=0}$ are Boolean formula in CNF over $\mathbf{X}\backslash \{X\}$ and $|\mathbf{X}|=n+1$, we can apply $H_n$: \begin{itemize}
            \item $\text{ProbDPLL}(G_{|X=1}, \sigma)$ terminates and returns $\text{PWMC}_{G_{|X=1}}^\sigma$.
            \item $\text{ProbDPLL}(G_{|X=0}, \sigma)$ terminates and returns $\text{PWMC}_{G_{|X=0}}^\sigma$.
        \end{itemize}
        In conclusion, we have:
        \begin{align*}
            \text{PWMC}_G^\sigma = & \text{PWMC}_{G_0}^\sigma + \text{PWMC}_{G_1}^\sigma \\
                                 = & (1-\sigma(X))\cdot \text{PWMC}_{G_{|X=0}}^\sigma + \\
                                   & \sigma(X)\cdot\text{PWMC}_{G_{|X=1}}^\sigma \\
                                 = & (1-\sigma(X))\cdot \text{ProbDPLL}(G_{|X=0}, \sigma) + \\
                                   & \sigma(X)\cdot\text{ProbDPLL}(G_{|X=1}, \sigma)
        \end{align*}
        Thus $\text{ProbDPLL}(G, \sigma)$ terminates and returns $\text{PWMC}_G^\sigma$.
     
\end{proof}

\section{DPNL Termination and Correctness}\label{appendix:proof:DPNLCorrectness}
In this Section we prove the termination and correctness of DPNL. 
\begin{theorem}
    For every
    \begin{itemize}
        \item PL problem $((\Omega, \tribe, \prob), \inputspace, \outputspace, F)$,
        \item PNL system $((\Vset{k})_{1\leq k \leq m}, (\Xvar{k})_{1\leq k \leq m}, (\hat{\distrib_{k}})_{1\leq k \leq m}, S)$ for the PL problem,
        \item valid oracle $O_S$ for $S$,
        \item $i \in \inputspace$, $o \in \outputspace$ and $v \in \setofvaluations{}$,
    \end{itemize}   
    $\text{DPNL}(i,o,v,O_S)$ terminates and returns $\prob(F(i) = o | w(i,v))$.
\end{theorem}

\begin{proof}
    We proceed by induction over the number of non-assigned variables $n$ of the valuation $v$, i.e, the size of $v^{-1}(\unknown)$.
    We define $H_n$ as: 
    
    \begin{quote}
    \emph{
    For every
    \begin{itemize}
        \item PL problem $((\Omega, \tribe, \prob), \inputspace, \outputspace, F)$,
        \item PNL system $((\Vset{k})_{1\leq k \leq m}, (\Xvar{k})_{1\leq k \leq m}, (\hat{\distrib_{k}})_{1\leq k \leq m}, S)$ for this problem,
        \item valid oracle $O_S: \setofvaluations{} \times \outputspace \rightarrow \{0,1, \unknown\}$ for $S$,
        \item $i \in \inputspace$, $o \in \outputspace$ and $v \in \setofvaluations{}$,
    \end{itemize} 
    such that $|v^{-1}(\unknown)| = n$, $\text{DPNL}(i,o,v,O_S)$ terminates and returns $\prob(F(i) = o | w(i,v))$.}
    \end{quote}
    
    \noindent We can prove $H_n$ for all $n \geq 0$ by induction:
        \paragraph{\underline{\normalfont{Base case $H_0$}} :} Let
        \begin{itemize}
            \item $((\Omega, \tribe, \prob), \inputspace, \outputspace, F)$ be a PL problem,
            \item $((\Vset{k})_{1\leq k \leq m}, (\Xvar{k})_{1\leq k \leq m}, (\hat{\distrib_{k}})_{1\leq k \leq m}, S)$ be a PNL system for this problem,
            \item $O_S: \setofvaluations{} \times \outputspace \rightarrow \{0,1, \unknown\}$ be a valid oracle for $S$,
            \item $i \in \inputspace$, $o \in \outputspace$ and $v \in \setofvaluations{}$,
        \end{itemize}
        such that $|v^{-1}(\unknown)| = 0$.
        
        Since $\forall k \in [1,m]$, $v[k] \neq \unknown$, $v$ is total and thus $tot(v) = \{v\}$.
        By definition of $w(i,o)$, $\forall w \in w(i,v)$, 
        \begin{center}
            $[\Xvar{1}(i)(w),...,\Xvar{k}(i)(w)] \in tot(v)$.
        \end{center}
        Therefore, for all $w \in w(i,v)$ and $k \in [1,m]$,
        \begin{center}
            $\Xvar{k}(i)(w) = v[k]$.
        \end{center}
        It follows, for all $w \in w(i,v)$,
        \begin{align*}
            F(i)(w) = & S(\Xvar{1}(i)(w),...,\Xvar{1}(i)(w)) \\
                    = & S(v[1],...,v[m]).
        \end{align*}

        There are two cases:
        \begin{itemize}
            \item Case $S(v[1],...,v[m]) = o$:
            
            In this case, for all $w \in w(i,v)$, $F(i)(w) = o$ thus $\prob(F(i) = o | w(i,v)) = 1$.

            Moreover, $O_S$ is a valid oracle so that $O_S(v,o) = 1$ because $v$ is total and $S(v[1],...,v[m]) = o$.
            It follows that $\text{DPNL}(i,o,v,O_S)$ terminates and returns $1$.

            \item Case $S(v[1],...,v[m]) \neq o$:
            
            In this case, for all $w \in w(i,v)$, $F(i)(w) \neq o$ thus $\prob(F(i) = o | w(i,v)) = 0$.
            Moreover, $O_S$ is a valid oracle so that $O_S(v,o) = 0$ because $v$ is total and $s(v[1],...,v[m]) \neq o$.
            It follows that $\text{DPNL}(i,o,v,O_S)$ terminates and returns $0$.

        \end{itemize}

        \paragraph{\underline{\normalfont{Induction step }$H_n \implies H_{n+1}$}:} Let $n\geq0$ and assume that $H_n$ holds. Let
        \begin{itemize}
            \item $((\Omega, \tribe, \prob), \inputspace, \outputspace, F)$ be a PL problem,
            \item $((\Vset{k})_{1\leq k \leq m}, (\Xvar{k})_{1\leq k \leq m}, (\hat{\distrib_{k}})_{1\leq k \leq m}, S)$ be a PNL system for this PL problem,
            \item $O_S: \setofvaluations{} \times \outputspace \rightarrow \{0,1, \unknown\}$ be a valid oracle for $S$,
            \item $i \in \inputspace$, $o \in \outputspace$ and $v \in \setofvaluations{}$,
        \end{itemize}
        such that $|v^{-1}(\unknown)| = n+1$.

        There are 3 cases about $O_S(v,o)$:
        \begin{itemize}
            \item Case $O_S(v,o) = 0$:
            
            In this case, $\text{DPNL}(i,o,v,O_S)$ terminates and returns $0$.
            We prove that $\prob(F(i) = o | w(i,v)) = 0$.

            Let $w \in w(i,o)$.

            By definition
            \begin{center}
                $[\Xvar{1}(i)(w),...,\Xvar{k}(i)(w)] \in tot(v)$
            \end{center}

            Moreover, since $O_S$ is a valid oracle and $O_S(v,o)=0$ for all $v' \in tot(v)$,
            \begin{center}
                $S(v'[1],...,v'[m]) \neq o$. 
            \end{center}

            Therefore, 
            \begin{center}
                $F(i)(w) = S(\Xvar{1}(i)(w),...,\Xvar{m}(i)(w)) \neq o$.
            \end{center}

            It follows that $\prob(F(i) = o | w(i,v)) = 0$.

            \item Case $O_S(v,o) = 1$:
            
            In this case, $\text{DPNL}(i,o,v,O_S)$ terminates and returns $1$.
            We prove that $\prob(F(i) = o | w(i,v)) = 1$.

            Let $w \in w(i,o)$.

            By definition
            \begin{center}
                $[\Xvar{1}(i)(w),...,\Xvar{m}(i)(w)] \in tot(v)$.
            \end{center}

            Moreover, since $O_S$ is a valid oracle and $O_S(v,o)=0$ for all $v' \in tot(v)$
            \begin{center}
                $S(v'[1],...,v'[m]) = o$
            \end{center}

            Therefore,
            \begin{center}
                $F(i)(w) = S(\Xvar{1}(i)(w),...,\Xvar{k}(i)(w)) = o$.
            \end{center}

            It follows that $\prob(F(i) = o | w(i,v)) = 1$.

            \item Case $O_S(v,o) = \unknown$:
            
            In this case, $\text{DPNL}(i,o,v,O_S)$ chooses $k\in [1,m]$ such that $v[k] = \unknown$.

            Since $v$ is partial, such a $k$ exists.
            Indeed, if $v$ was a total valuation, $O_S(v,o)$ would return $0$ or $1$.

            The events $(\Xvar{k}(i) = y)_{y \in \Vset{k}}$ form a partition of $\Omega$.
            By the law of total probability, $\prob(F(i) = o | w(i,v))$ is equal to
            \begin{center}
                $\Sigma_{y\in \Vset{k}}\prob(\Xvar{k}(i)=y|w(i,v))\cdot\prob(F(i)=o|w(i,v)\cap \Xvar{k}(i)=y)$.
            \end{center}

            By definition, $w(i,v)$ only depends on the $\Xvar{l}(i)$ such that $v[l] \neq \unknown$, $v[k] = \unknown$ and the $(\Xvar{l}(i))_{1\leq l \leq m}$ are independent by the definition of the PNL system.
            Therefore, $w(i,v)$ is independent of $\Xvar{k}(i)$ for all $y \in \Vset{k}$:
            \begin{center}
                $\prob(\Xvar{k}(i)=y|w(i,v)) = \prob(\Xvar{k}(i)=y)$.
            \end{center}

            Moreover, by definition of $v_{v[k]\gets y}$ and $w(i,v_{v[k]\gets y})$
            \begin{center}
                $w(i,v_{v[k]\gets y}) = w(i,v) \cap (\Xvar{k}(i)=y)$
            \end{center}
            because $v[k] = \unknown$.

            It holds
            \begin{center}
                $\prob(F(i)=o|w(i,v)\cap \Xvar{k}(i)=y) = \prob(F(i)=o|w(i,v_{v[k]\gets y}))$,
            \end{center}
            
            Finally, since $v[k] = \unknown$ and $|v^{-1}(\unknown)| = n+1$,
            \begin{center}
                $|v_{v[k]\gets y}^{-1}(\unknown)| = n$
            \end{center}
            Consequently, we can apply $H_n$. 
            For all $y \in \Vset{k}$, $\text{DPNL}(i,o,v_{v[k]\gets y}, O_S)$ terminates and returns $\prob(F(i)=o|w(i,v_{v[k]\gets y}))$.

            We conclude that DPNL terminates and returns $\prob(F(i) = o | w(i,v))$.

    \end{itemize}
     
\end{proof}

\section{Heuristics for choosing k in DPNL} 

In DPNL, choosing $k$ with its associated variable $\Xvar{k}$ (see line~\ref{alg:DPNL:choosek} of Algorithm~\ref{alg:DPNL}), as for ProbDPLL, determines the number of recursive calls.
The objective is to choose $k$ such that the execution tree of DPNL is as small as possible.
The heuristic is designed based on the specific oracle.
However, when the oracle is complete, the following heuristic is intuitively interesting:
\begin{itemize}
    \item
    When the oracle is complete and answers $\unknown$ this means that there exists a total sub-valuation $v_0$ of $v$ such that $S(v[1],...,v[m]) \neq o$ and a total sub-valuation $v_1$ of $v$ such that $S(v[1],...,v[m]) = o$.
    It is often straightforward to modify an oracle to return these two sub-valuations instead of $\unknown$ by taking inspiration from the proof of completeness of the oracle, which often explicitly constructs these sub-valuations.
    \item The heuristic should choose $k$ so that $v[k] = \unknown$ and $v_0[k] \neq \unknown$ or $v_1[k] \neq \unknown$.
\end{itemize}

\section{MNIST-N-SUM Oracle}\label{appendix:MNISTAddOracle}

In MNIST-N-SUM task as defined in exampe~\ref{ex:mnist_sum}, $S$ is the sum algorithm which operates digit by digit, computing the digits of the result from right to left (see Algorithm~\ref{alg:Addition}).
We can take advantage of this to design an efficient oracle which is able to prune the search space (see Algorithm~\ref{alg:AdditionOracle}).
This oracle computes the result of the sum digit by digit from right to left using $v$, and tests whether it matches the expected result $o$.
If it encounters that the result differs, it returns $0$, and if it encounters that $v[k]$ is $\unknown$, it returns unknown.
Note that by choosing $k$ in sequential order from right to the left ($N,N+N, N-1,N+(N-1), ... , 1, N+1$) in DPNL, Algorithm~\ref{alg:AdditionOracle} acts like a complete oracle. 
\begin{algorithm}[h]
    \caption{Addition$(d_1,...,d_{2\cdot N})$}
    \label{alg:Addition}
    \textbf{Input}: $2\cdot N$ digits representing two numbers $n_1 = d_1...d_N$ and $n_2 = d_{N+1}...d_{N+N}$ \\
    \textbf{Output}: $r = r_0...r_{N} = n_1 + n_2$
    \begin{algorithmic}[1] 
        \STATE $r_0...r_{N}$ $\gets$ $0,...,0$
        \STATE $carry$ $\gets$ $0$
        \FOR {$i \in [N,N-1,...,1]$}
            \STATE $d$ $\gets$ $carry+d_i+d_{N+i}$
            \STATE $r_i$ $\gets$ $d\text{ mod }10$
            \STATE $carry$ $\gets$ inferior integer part of $d/10$
        \ENDFOR
        \STATE $r_0$ $\gets$ $carry$
        \STATE \textbf{return} $r = r_0...r_N$
    \end{algorithmic}
\end{algorithm}

\begin{algorithm}[h]
    \caption{AdditionOracle$(v,r)$}
    \label{alg:AdditionOracle}
    \textbf{Input}: A valuation $v$ of size $2\cdot N$ with values in $[0,9] \cup \{\unknown\}$, an output result $r \in [0,2\cdot 10^N-1]$ \\
    \textbf{Output}: A valid oracle output for the symbolic function 
    \begin{algorithmic}[1] 
        \STATE $r_0...r_{N}$ $\gets$ $r$ (Extract the base $10$ representation of $r$)
        \STATE $carry$ $\gets$ $0$
        \FOR {$i \in [N,N-1,...,1]$}
            \STATE \textbf{if} $v[i]$ or $v[N+i]$ are $\unknown$ \textbf{return} $\unknown$
            \STATE $d$ $\gets$ $carry+v[i]+v[N+i]$
            \STATE \textbf{if} $d \text{ mod } 10 \neq r_i$ \textbf{return} $0$
            \STATE $carry$ $\gets$ inferior integer part of $d/10$
        \ENDFOR
        \STATE \textbf{if} $carry = r_0$ \textbf{return} $1$ \textbf{else} \textbf{return} $0$
    \end{algorithmic}
\end{algorithm}
We now prove that the AdditionOracle in Algorithm~\ref{alg:AdditionOracle} is a valid oracle (as defined in Section~\ref{section:oracle}) to compute the addition of N-digits as defined in Algorithm~\ref{alg:Addition}.
\begin{theorem}
    For all $N\geq 0$, AdditionOracle is a valid oracle for Addition.
\end{theorem}

\begin{proof}
    Let $N\geq 0$.
    We check each property of a valid oracle for Addition:
    \begin{itemize}
        \item AdditionOracle returns a result in $\{0,1,\unknown\}$.

        \item Let $v \in \setofvaluations{}$ and $o \in [0,2\cdot 10^N]$.
        Suppose that the AdditionOracle returns $0$.
        It follows that $\operatorname{log}_{10}(o) \leq (N+1)$ and we can extract the base $10$ notation of $r_0...r_N$ which is the output $o$ for $N+1$ digits.
        Moreover, the for loop performs the addition algorithm on $v[1]...v[N]$ and $v[N+1]...v[N+N]$.
        If the for loop ends, the valuation is total, but the digit of the result does not correspond to this of $o$.
        It follows that for all $v' \in tot(v)$, $v' = v$ and Addition$(v'[1],...,v'[N+N]) \neq o$. 
        Otherwise, for each digit in the result calculated from right to left, it checks whether it is equal to the corresponding digit in $o$ and fails for $i_0$.
        This has two implications: 
        \begin{itemize}
            \item When the algorithm has reached this point, for all $i \in [N,N-1,...,i_0]$, $v[i] \neq \unknown$ and $v[N+i] \neq \unknown$.
            \item The $i_0$-th digit of the result does not correspond to $r_{i_0}$.
            It results that the sum differs from $o$, regardless of the values $v[i]$ and $v[N+i]$ with $i \in [i_0-1,...,1]$.
        \end{itemize}
        We define the valuation $v_0$ such that for all $i \in [1,N]$,
        \begin{itemize}
            \item[] $v_0[i] = 
            \begin{cases}
                v[i] & \text{if } i \geq i_0, \\
                \unknown & \text{else}
            \end{cases}$ and 
            \item[] $v_0[N+i] = 
            \begin{cases}
                v[N+i] & \text{if } i \geq i_0, \\
                \unknown & \text{else}.
            \end{cases}$
        \end{itemize}
        
        Regarding the second point it holds that Addition$(v'[1],...,v'[N+N]) \neq o$ for all $v' \in tot(v_0)$, 
        and, regarding the first point, $tot(v) \subset tot(v_0)$.
        It can be concluded that Addition$(v'[1],...,v'[N+N]) \neq o$ for all $v' \in tot(v)$. 

        \item Let $v \in \setofvaluations{}$ and $o \in [0,2\cdot 10^N]$. Suppose that the oracle returns $1$.
        Then $\operatorname{log}_{10}(o) \leq (N+1)$ and we can extract the base $10$ notation of $r_0...r_N$ for $o$ on $N+1$ digits.
        Moreover, the for loop performs the addition algorithm on $v[1]...v[N]$ and $v[N+1]...v[N+N]$ and terminates.
        It results that $v$ is total and all digits of the result correspond to those of $o$, hence for all $v' \in tot(v)$, $v' = v$ and $\text{Addition}(v'[1],...,v'[N+N]) = o$
    \end{itemize}
\end{proof}

\section{Logic PNL systems and Oracles} 

We consider the definitions, notations and algorithms in Section~\ref{section:ObtainingOracleFromLogic}.

\subsection{Logic PNL Systems Correctness}\label{proof:LogicPNLSystemCorrectness}

\begin{theorem}
    Consider a logic $\Logic$, a domain $\Delta$ and a closed formula $\phi$.
    We assume the existence of an axiom set $A = \{\alpha_1, ..., \alpha_n\}$ and a set of closed formulas $\phi_1,...,\phi_m$ that decide $\phi$ and are coherent with respect to $A$.
    Furthermore, $\bigl((\{0, 1\})_{1\leq k \leq m}, (B_{\Logic, \Delta}^{\phi_k})_{1\leq k \leq m}, (\hat{\distrib}_k)_{1\leq k \leq m}, S\bigr)$
    is a PNL system for $P_{\Logic, \Delta}^{\phi}$.
\end{theorem}

\begin{proof}
    Consider a logic $\Logic$, a domain $\Delta$ and a closed formula $\phi$.
    Let us consider the following hypothesis (denoted as H).
    \begin{itemize}
        \item The existence of $A = \{\alpha_1, ..., \alpha_n\}$ such that all $\alpha_i$ are true in $\Gamma_{\Logic, \Delta}$, i.e, for all $i \in [1,n]$:
        \begin{itemize}
            \item[(H$_i$)] $\bigcup_{[.]_\Func \in \inputspace_{\Logic, \Delta}}(B_{\Logic,\Delta}^{\alpha_i}([.]_\Func) = 0) = \emptyset$
        \end{itemize}
        \item The existence of a set of closed formulas $\phi_1,...,\phi_m$ such that:
        \begin{itemize}
            \item[(H1)] $(P_{\Logic, \Delta}^{\phi_k})_{1\leq k \leq m}$ are independent logic PL problems
            \item[(H2)] For all $\phi'_1,...,\phi'_m$ such that $\phi'_k = \phi_k$ or $\phi'_k = \neg \phi_k$, $A \cup \{ \phi'_1,...,\phi'_m \} \vdash \phi$ or $A \cup \{ \phi'_1,...,\phi'_m \} \vdash \neg\phi$.
            \item[(H3)] For all $\phi'_1,...,\phi'_m$ such that $\phi'_k = \phi_k$ or $\phi'_k = \neg \phi_k$, $A \cup \{ \phi'_1,...,\phi'_m \} \not \vdash \bot$.
        \end{itemize}
    \end{itemize}
    By definition and by the assumptions :
    \begin{itemize}
        \item $\forall k \in [1,m]$, $P_{\Logic, \Delta}^{\phi_k}$ is a PL problem as previously defined
        \item $(P_{\Logic, \Delta}^{\phi_k})_{1\leq k \leq m}$ are independent logic PL problems (by H1)
        \item $\forall k \in [1,m]$, $\hat{\distrib_k}$ is a neural network which solves $P_{\Logic, \Delta}^{\phi_k}$, by definition. 
        \item $S : \{0,1\}^m \rightarrow \{0,1\}$.
    \end{itemize}
    It still remains to prove that $B_{\Logic, \Delta}^{\phi}$ is equal to $\Sfunction \circ (B_{\Logic, \Delta}^{\phi_1},...,B_{\Logic, \Delta}^{\phi_m})$.
    
    Let $[.]_\Func \in \inputspace_{\Logic, \Delta}$ and $[.]_\Pred \in \Omega_{\Logic, \Delta}$.
    We note:
    \begin{itemize}
        \item For all $k \in [1,m]$, $x_k = B_{\Logic, \Delta}^{\phi_k}([.]_\Func)([.]_\Pred)$
        \item $T = A \cup \{\phi_k | x_k =1\} \cup \{\neg \phi_k | x_k = 0 \}$
    \end{itemize}
    In order to prove that $B_{\Logic, \Delta}^{\phi}([.]_\Func)([.]_\Pred) = S(x_1,...,x_m)$ we consider two cases: 
    \begin{itemize}
        \item \underline{If $S(x_1,...,x_m) = 1$} :
        
        By definition of $S$, $T \vdash \phi$.
        
        For all interpretations $I$, if $I$ satisfies $T$, that is:
        \begin{itemize}
            \item For all $i \in [1, n]$, $[\alpha]_I = 1$
            \item For all $k \in [1,m]$, $[\phi_k]_I = x_k$
        \end{itemize}
        it follows by soundness that $[\phi]_I = 1$.
        
        Moreover, 
        \begin{itemize}
            \item For all $i \in [1, n]$, $[\alpha]_{(\Delta, [.]_\Func, [.]_\Pred)} = 1$ (by H$_i$)
            \item For all $k \in [1,m]$, $x_k = B_{\Logic, \Delta}^{\phi_k}([.]_\Func)([.]_\Pred) = [\phi_k]_{(\Delta, [.]_\Func, [.]_\Pred)}$ by definition.
        \end{itemize}
        Therefore, $B_{\Logic, \Delta}^{\phi}([.]_\Func)([.]_\Pred) = [\phi]_{(\Delta, [.]_\Func, [.]_\Pred)} = 1 = S(x_1,...,x_m)$.

        \item \underline{If $S(x_1,...,x_m) = 0$} :

        By definition of $S$, $T \not \vdash \phi$.
        Since $T \not \vdash \phi$, applying H2 results in $T \vdash \neg \phi$.
            
        For all interpretation $I$, if $I$ satisfy $T$, that is:
        \begin{itemize}
            \item For all $i \in [1, n]$, $[\alpha]_I = 1$
            \item For all $k \in [1,m]$, $[\phi_k]_I = x_k$
        \end{itemize}
        it follows by soundness $[\neg \phi]_I = 1$.
        By applying the semantics of $\neg$ we can derive $[\phi]_I = 1 - [\neg \phi]_I = 0$.
        Moreover,
        \begin{itemize}
            \item For all $i \in [1, n]$, $[\alpha]_{(\Delta, [.]_\Func, [.]_\Pred)} = 1$ by (H$_i$)
            \item For all $k \in [1,m]$, $x_k = B_{\Logic, \Delta}^{\phi_k}([.]_\Func)([.]_\Pred) = [\phi_k]_{(\Delta, [.]_\Func, [.]_\Pred)}$ by definition.
        \end{itemize}
        It results that  $B_{\Logic, \Delta}^{\phi}([.]_\Func)([.]_\Pred) = [\phi]_{(\Delta, [.]_\Func, [.]_\Pred)} = 0 = S(x_1,...,x_m)$.
            
    \end{itemize}

\end{proof}

\subsection{Logic Oracle Validity and Completeness}\label{proof:LogicOracleValidityCompletness}
We now proof validity and completeness of the logic oracle in Algorithm~\ref{alg:LogicOracle}.

\begin{theorem}
    The logic oracle is a valid oracle for $S$.
\end{theorem}

\begin{proof}
    We denote $O(v,o)$ the as the result of the logic oracle on $v\in \setofvaluations$ and $o \in\{0,1\}$ (in Algorithm~\ref{alg:LogicOracle}).
    We check that Algorithm~\ref{alg:LogicOracle} satisfies every property of a valid oracle for $S$:
    \begin{itemize}
        \item Since $res$ in the algorithm takes values in $\{1,0,\unknown\}$, the algorithm returns values in $\{0,1,\unknown\}$.

        \item Let $v \in \setofvaluations$ and $o \in \{0,1\}$ and assume that $O(v,o) = 0$.

        Given that $O$ satisfies $O(v,0) = 1-O(v,1)$ and that $res$ in the algorithm is inverted for $o=0$, we can suppose without loss of generality that $o=1$.
        We now want to prove that for all $v' \in tot(v)$, $S(v'[1],...,v'[m]) = 0$.

        Let $v' \in tot(v)$.

        We note $T = A \cup \{\phi_k | v[k] = 1\} \cup \{\neg \phi_k | v[k] = 0\}$ and $T' = A \cup \{\phi_k | v'[k] = 1\} \cup \{\neg \phi_k | v'[k] = 0\}$.

        Since $O(v,o)=0$ and we assumed that $o=1$, this implies that $res$ has been set to $0$ which results in $T \vdash \neg \phi$.

        Moreover, $v'$ is a sub-valuation of $v$ so that $T \subset T'$ hence $T' \vdash \neg \phi$ because the proof tree of $\neg \phi$ in $T$ is also a valid proof tree in $T'$.

        In addition, by H3, $T' \not \vdash \bot$ so that $T' \not \vdash \phi$ since $T' \vdash \neg \phi$

        By definition of $S$ it follows $S(v'[1],...,v'[m]) = 0$.

        \item Let $v \in \setofvaluations$ and $o \in \{0,1\}$ and assume that $O(v,o) = 1$.

        Given that $O$ satisfies $O(v,0) = 1-O(v,1)$ and that $res$ is inverted for $o=0$, we can suppose without loss of generality that $o=1$.
        We thus want to prove that for all $v' \in tot(v)$, $S(v'[1],...,v'[m]) = 1$.

        Let $v' \in tot(v)$.

        We note $T = A \cup \{\phi_k | v[k] = 1\} \cup \{\neg \phi_k | v[k] = 0\}$ and $T' = A \cup \{\phi_k | v'[k] = 1\} \cup \{\neg \phi_k | v'[k] = 0\}$.

        Since $O(v,o)=1$ and we assumed that $o=1$, this implies that $res$ has been set to $1$ which results in $T \vdash \phi$.

        Moreover, $v'$ is a sub-valuation of $v$ so that $T \subset T'$ hence $T' \vdash \phi$ because the proof tree of $\phi$ in $T$ is also a valid proof tree in $T'$.

        By definition of $S$ it follows $S(v'[1],...,v'[m]) = 1$.

    \end{itemize}
    
\end{proof}

\begin{theorem}
    The logic oracle describes a complete oracle for $S$.
\end{theorem}

\begin{proof}
    We denote $O(v,o)$ as the result of Algorithm~\ref{alg:LogicOracle} on $v\in \setofvaluations$ and $o \in\{0,1\}$.
    Since we have already proven the validity, it remains to prove that for all $v \in \setofvaluations$ and $o \in \{0,1\}$, if $O(v,o)= \unknown$ there exists $v_0, v_1 \in tot(v)$ such that $S(v_0[1],...,v_0[m]) = 0$ and $S(v_1[1],...,v_1[m]) = 1$.
    
    Let $v \in \setofvaluations$ and $o \in \{0,1\}$ such that $O(v,o)= \unknown$.
    We note $T=A \cup \{\phi_k | v[k] = 1\} \cup \{\neg \phi_k | v[k]=0 \}$.

    Since $O(v,o)= \unknown$, it means that $T \not \vdash \phi$ and $T \not \vdash \neg \phi$.
    We can thus apply the completeness theorem and consider interpretations $I_0$ and $I_1$ that satisfy $T$ such that $[\phi]_{I_0} = 0$ and $[\phi]_{I_1} = 1$.
    
    For $i \in \{0,1\}$, we note $v_i=[[\phi_1]_{I_i},...,[\phi_m]_{I_i}]$ and $T_i = A \cup \{\phi_k | v_i[k] = 1\} \cup \{\neg \phi_k | v_i[k]=0 \}$.
    Since $I_i$ satisfies $T$, we conclude that $v_i \in tot(v)$.
    We now show that $S(v_i[1],...,v_i[m]) = i$:
    \begin{itemize}
        \item \underline{$i=0$}:
        
        By contradiction, we assume that $S(v_0[1],...,v_0[m]) = 1$, which implies $T_0 \vdash \phi$.
        By definition of $I_0$, $v_0$ and $T_0$, $I_0$ satisfies $T_0$.
        Given completeness, it results that $[\phi]_{I_0} = 1$, which contradicts $[\phi]_{I_0} = 0$.

        It follows that $S(v_0[1],...,v_0[m]) = 0$.

        \item \underline{$i=1$}:
        
        By contradiction, we assume that $S(v_1[1],...,v_1[m]) = 0$.
        It means that $T_1 \not \vdash \phi$ hence $T_1 \vdash \neg \phi$ because we assume that $\phi_1,...,\phi_m$ decide $\phi$ w.r.t. $A$ when defining the PNL system.

        By definition of $I_1$, $v_1$ and $T_1$, $I_1$ satisfies $T_1$.
        Given completeness, it results that $[\phi]_{I_1} = 1 - [\neg \phi]_{I_1} = 0$, which contradicts $[\phi]_{I_1} = 1$.

        It follows that $S(v_1[1],...,v_1[m]) = 1$.

    \end{itemize}

\end{proof}

\section{Application to DeepProbLog} \label{appendix:applicationtodeepproblog}
This section discusses how a logic oracle can be used in the context of programs in DeepProbLog. 

\subsection{Formal definitions}
We start by giving formal the formal definitions of a Problog program \cite{NeSy:ProbLog:DeRaedt2007} and a DeepProblog program \cite{NeSy:DeepProbLog:Manhaeve2018}.
\paragraph{ProbLog program.}
A ProbLog program is defined as a set $T = \{p_1 : c_1,..., p_n : c_n\}$ of rules $c_k$ without negation annotated by their probability $p_k$.
It defines a probability distribution over logic programs $L \subseteq L_T = \{c_1,\hdots, c_n\}$ as follows: 
\begin{align*}
    P_T(L) := & \prod_{c_k \in L}p_k \cdot \prod_{c_k \notin L}(1-p_k)
\end{align*}
For every query $q$ the probability of success of $q$ is defined as:
\begin{align*}
    P_T(q) := & \sum_{L \subseteq L_T, L \vDash q} P_T(L)
\end{align*}

\paragraph{DeepProbLog program.}
For the sake of clarity, we use different notations from \cite{NeSy:DeepProbLog:Manhaeve2018} to define a DeepProbLog program. 
Formally, a DeepProbLog program over logic $\Logic$ (as previously defined) is a set $D = T \cup T_N$ where:
\begin{itemize}
    \item $T = \{ c_1, ..., c_n \}$ is a Prolog program over $\Logic$ without negation,
    \item $T_N = \{\nn_1 : r_1,..., \nn_n : r_m\}$ is the set of neural predicates in the program
    \item $r_k$ is a Prolog rule without negation with $l_k$ ground terms $t_k^1, ... ,t_k^{l_k}$
    \item $\nn_k =  (m_k, t_k^1, ... ,t_k^{l_k})$ where $m_k$ is a neural network that maps the interpretations of the $t_k^1, ... ,t_k^{l_k}$ as element of $\Delta$ to the probability of $r_k$.
\end{itemize}
For all $[.]_\Func$, we define the interpretation of $D$ as the DeepProbLog program:
\begin{center}
    $[D]_\Func = \{1:c_k| 1\leq k \leq n\} \cup  \{m_k([t_k^1]_\Func,...,[t_k^{m_k}]_\Func) : r_k | 1 \leq k \leq m\}$,
\end{center}
where $[t_k^l]_\Func$ is the interpretation of $t_k^l$ in $\Delta$ by recursively applying $[.]_\Func$ to the functions of $t_k^l$.

\paragraph{DeepProbLog PNL system.}
Consider $P_{\Logic, \Delta, q}$ and $D=T\cup T_N$ as defined above.
If
\begin{itemize}
    \item $q$ corresponds to a Prolog query,
    \item $(P_{\Logic, \Delta}^{r_k})_{1\leq k \leq m}$ are independent PL problems,
    \item $c_i \in T$ is true, i.e, $\underset{[.]_\Func \in \inputspace_{\Logic, \Delta}}{\bigcup}(B_{\Logic,\Delta}^{c_i}([.]_\Func) = 0) = \emptyset$.
\end{itemize}
$D$ corresponds to a PNL system for $P_{\Logic, \Delta}^q$.
Indeed, $((\{0,1\})_{1\leq k \leq m}, (B_{\Logic, \Delta}^{r_k})_{1\leq k \leq m}, (\hat{\distrib}_k)_{1\leq k \leq m}, S)$ is a PNL system where
\begin{itemize}
    \item $\hat{\distrib}_k([.]_\Func)(1) = 1-\hat{\distrib}_k([.]_\Func)(0) = m_k([t_k^1]_\Func,...,[t_k^l]_\Func)$,
    \item $S(x_1,...,x_m) = \begin{cases}
            1 \text{ if } T \cup \{r_k | x_k=1\} \vDash q, \\
            0 \text{ else}
        \end{cases}$.x
\end{itemize}

\paragraph{Inference in DeepProbLog.}
Given $[.]_\Pred \in \inputspace_{\Logic, \Delta}$, DeepProbLog computes $S$ in the form of a logical formula of the $x_1,...,x_m$ in DNF.
Then, it computes $\text{PWMC}_S^\sigma$ with the probability $\sigma(x_k) = m_k([t_k^1]_\Pred,...,[t_k^{l_k}]_\Pred)$ of $r_k$.
By definition of $P_{[D]_\Func}$ and $S$, $P_{[D]_\Func}(q)$ is equal to $\text{PWMC}_S^\sigma$.

\subsection{Oracle for DeepProbLog}
Building on a Prolog solver, Algorithm~\ref{alg:DPLOracle} provides an oracle for the logical provenance.
Its validity is based on the fact that the program has no negation.
It follows that if we cannot prove $q$ by adding all rules $r_k$ such that $v[k]=\unknown$, then we can directly return $0$.
If, in addition, the program has no functor, the Prolog solver is polynomial-time. It results that its oracle is also polynomial-time.

\begin{algorithm}[tb]
    \caption{DeepProbLog Oracle}
    \label{alg:DPLOracle}
    \textbf{Input}: $v \in \setofvaluations{}$, $o \in \outputspace$ \\
    \textbf{Output}: A valid oracle output for the function $S$ \\
    \begin{algorithmic}[1] 

        \IF {$R \cup \{r_k | v[k] = 1\} \vdash q$}
            \STATE $res \gets 1$
        \ELSIF {$R \cup \{r_k | v[k] = 1 \lor v[k] = \unknown\} \vdash q$}
            \STATE $res \gets \unknown$.
        \ELSE
            \STATE $res \gets 0$.
        \ENDIF

        \IF {$res \neq \unknown \land o = 0$}
            \STATE $res \gets 1-res$.
        \ENDIF

        \STATE \textbf{return} $res$.
        
    \end{algorithmic}
\end{algorithm}

\begin{example}[Graph Reachability]
    \label{ex:GraphReach}
    Graph reachability queries are the key ingredients in numerous tasks such as semi-supervised classification in citation networks \cite{NeSy:DeepStochLog:Winters2022}.
    A probabilistic directed graph consists of the following elements, denoted as follows 
    \begin{itemize}
        \item $E = \{e_k | 1 \leq k \leq N\}$ describes the set of nodes
        \item $G(e_i,e_j)$ denotes an edge from node $e_i$ to node $e_j$, ($e_i, e_j \in E$).
        \item $\nn_G(e_i,e_j)$ is the probability that the edge $G(e_i,e_j)$ is present in the graph.
        \item Reachability is defined as $R(e_i)$. 
        A node $e_i$ is reachable from node $e_1$ if $R(e_i)$ it true. 
    \end{itemize}
    For this problem, we are interested in the probability that a node $e_N$ is reachable from a starting node $e_1$ and formulate this as query $q = R(e_N)$.
    The full logical provenance formula for a query $q$ is denoted as $LP_q$.
    Consider the following DeepProbLog program $D = T_N \cup T$ for graph reachability:
    \begin{itemize}
        \item $T_N = \{\nn_{i,j}: G(e_i,e_j). | 1 \leq i,j \leq N\}$.
        \item $T = \{ R(e_1). , R(Y) :- R(X), G(X,Y).\}$
    \end{itemize}
    This program contains neither negation nor functions.
    Therefore, the DeepProbLog oracle in algorithm \ref{alg:DPLOracle} is polynomial-time.
    However, 
    \begin{itemize}
        \item for each path $e_1\rightarrow e_{i_1} \rightarrow ... \rightarrow e_{i_k} \rightarrow e_N$ (with intermediate nodes $e_{i_1}, \hdots, e_{i_k}$) the corresponding clause $G(e_1,e_{i_1}) \land ... \land G(e_{i_k},e_N)$ implies $LP_q$,
        \item and $LP_q$ implies the existence of such a path
    \end{itemize}
    Therefore, $LP_q$ is equivalent to the disjunction of clauses that correspond to the path from $e_1$ to $e_N$. Moreover, \begin{itemize}
        \item for each path we can potentially remove loops and obtain a path whose corresponding clause is included in the original clause (before removing loops),
        \item and if we remove a literal from a clause corresponding to a loop-free path it does not correspond to a path, hence it does not imply $LP_q$.
    \end{itemize} Therefore, $LP_q$ is exactly the disjunction of clauses that correspond to the loop-free path. Choosing a loop-free path from $e_1$ to $e_N$ in a complete graph is equivalent to: \begin{itemize}
        \item Choosing the number of intermediate node $i\in[0,|E|-2]$
        \item Choosing $i$ nodes from the set of remaining $|E|-2$ nodes in the graph
        \item Choosing a permutation of these $i$ nodes that correspond to the order in which they are visited
    \end{itemize}
    Therefore, $LP_q$ contains exactly $\sum_{i=0}^{|E|-2} \binom{|E|-2}{i}\cdot i!$ clauses.
    In conclusion, since $|E|-2 \underset{|E|\rightarrow \infty}{\sim} \sqrt{|T|} $, $|LP_q| \underset{|E|\rightarrow \infty}{\sim} (\sqrt{|T|})!$

    \end{example}

\section{Modeling Probabilistic Choices}
\label{appendix:nad}
This sections describes how probabilistic choices are treated in ProbLog \cite{NeSy:ProbLog:DeRaedt2007} (and in DeepProbLog \cite{NeSy:DeepProbLog:Manhaeve2018}), as detailed in \cite{Article:DeRaedt:probabilistic-logic-programming}.
Probabilistic choices are categorical random variables with mutually exclusive outcomes. 
The representation with probabilistic facts is not sufficient as these are considered to be independent.
In contrast, the outcomes of probabilistic choices are dependent. 
ProbLog uses \emph{annotated disjunctions (ADs)} to model probabilistic choices. 
Based on ADs, DeepProbLog uses \emph{neural annotated disjunctions (nADs)}, which can be represented by the activated classification function of a neural network. 

Probabilistic choices for ADs are transformed to (1) a set of probabilistic facts and (2) deterministic clauses: 
\begin{itemize}
    \item $\tilde{p}_i:: sw_{id}(h_i, v_1, \hdots, v_f)$ is the set of probabilistic facts.
    \item The deterministic clauses are 
    \begin{equation*}
        \footnotesize
    \begin{aligned}
        h_i & :- b_1, \ldots, b_m, not \left(sw\_id\left(h_1, v_1, \ldots, v_f\right)\right), \ldots, \\
        & \quad not \left(sw\_id\left(h_{i-1}, v_1, \ldots, v_f\right)\right), sw\_id\left(h_i, v_1, \ldots, v_f\right),
    \end{aligned}
    \end{equation*}
    where $v_1, \hdots v_f$ are the free variables in the body of the AD and $sw\_id$ is a switch variable with an identifier for the respective AD.
\end{itemize}
The probability $\tilde{p_i}$ is defined as 
$$
\tilde{p}_i:=\left\{\begin{array}{ll}
p_i \cdot\left(1-\sum_{j=1}^{i-1} p_j\right)^{-1} & \text { if } p_i>0 \\
0 & \text { if } p_i=0
\end{array} .\right.
$$
To recover the original probabilities \(p\) from the transformed values \(\tilde{p}\), one can initialize \(p_1 := \tilde{p}_1\) and iteratively compute \(p_i\) for \(i = \{2, 3, \ldots, n\}\) using the formula
\[
p_i := \tilde{p}_i \cdot \left(1 - \sum_{j=1}^{i-1} p_j\right).
\]

\begin{example}[ADs for MNIST Classification]
    Consider the example of MNIST classification, where the variables indicating that an image $I$  contains a certain digit in $[0,9]$ are modeled as probabilistic facts in ProbLog, e.g. 
    \[
    p_0 :: \mathrm{digit}(I, 0), \ p_1 :: \mathrm{digit}(I, 1), \dots, \ p_9 :: \mathrm{digit}(I, 9).
    \]
    This implies that the probability of all digits being simultaneously valid for an image equals \( \prod_{i=0}^9 p_i \).
    In contrast, since the variables are mutually exclusive (and therefore dependent), the probability of this event should be zero. 
    To this end, the digits can be modeled as probabilistic choices with ADs. 
    For the MNIST Classification task, the AD can be represented as follows: 
    
    $$\mathrm{digit}(I, 0) :- sw(\mathrm{digit}(I,0))$$
    $$\mathrm{digit}(I, 1) :- not(sw(\mathrm{digit}(I,0))), sw(digit(I,1))$$
    $$\mathrm{digit}(I, 2) :- not(sw(\mathrm{digit}(I,0))), \hdots, sw(\mathrm{digit}(I,2))$$
    $$\hdots$$
    $$\mathrm{digit}(I,9) :- not(sw(\mathrm{\mathrm{digit}}(I,0))), \hdots, sw(\mathrm{digit}(I,9))$$
\end{example} 

\section{Experimental Details}
\label{appendix:experimental_details}
\subsection{Hardware} All experiments are executed on an Apple Macbook M1 Max with 64GB RAM.
\subsection{Implementation Details}
DPNL is implemented in Python and based on Pytorch \cite{paszke2019pytorch}. 
We use Weights and Biases \cite{wandb} as experiment tracking tool. 
The obtained results with DPNL will be made publicly available.\footnotemark[1]

\subsection{Hyperparameters.} 
\label{appendix:hyperparameters}
The hyperparameters used in the experiments are summarized in Table~\ref{tab:hyperparameters}.
We keep the hyperparameters constant across the experiments with different numbers of digits. 
Further, we conduct 10 independent runs with different random seeds for DPNL. 
\begin{table}[h]
    \centering
    \begin{tabular}{|l|l|}
    \hline
    \textbf{Parameter} & \textbf{value}\\
    \hline
    \textbf{\# epochs}  & 1 \\
    \textbf{Batch size} & 2 \\
    \textbf{Learning rate} & 1e-03\\
    \textbf{Classifier} & MNIST Classifier\\
    \textbf{Loss} & BCE Loss\\
    \hline
    \end{tabular}
    \caption{Hyperparameters used for DPNL and baselines.\label{tab:hyperparameters}}
\end{table}
\paragraph{MNIST Classifier.}
\label{mnist_classifier}
We follow \cite{NeSy:DeepProbLog:Manhaeve2018} and use the following neural network architecture as digit classifier. 
It consists of two convolutional layers with kernel size 5 and output dimension 6 and 16 (\texttt{nn.Conv2d(1, 6, 5)} and \texttt{nn.Conv2d(6, 16, 5),}) followed by Max pooling and Relu activation (\texttt{nn.MaxPool2d(2, 2)}).
A classifier with three linear layers is stacked which have the dimensions 120 and 84 respectively. 
The output dimension corresponds to the number of digits and equals 10. 

\subsection{Baselines} To obtain the results with the baselines, we used the source code from the following publicly available code repositories: 
\begin{itemize}
    \item DeepProbLog \& DPLA$^*$: 
    \item []\url{https://github.com/ML-KULeuven/deepproblog}
    \item A-Nesi: \url{https://github.com/HEmile/a-nesi}
    \item Scallop: \url{https://github.com/scallop-lang/scallop}
\end{itemize}
The code with which we ran the baselines and measured the runtimes is available.\footnote[1]{\url{https://anonymous.4open.science/r/dpnll_19FE/}}
\paragraph{A-Nesi.}
We conduct five independent runs with different random seeds for each experiment. 
Table~\ref{tab:hyperparameters:anesi} shows the hyperparameters used for configurations used for the different variants of A-Nesi \cite{NeSy:A-NeSI:vanKrieken2023}. 
\begin{table}[H]
    \small
    \centering
    \begin{tabular}{|c|c|c|c|}
    \hline
    \textbf{Key} & \textbf{Explain} & \textbf{Predict Only} & \textbf{Prune} \\ \hline
    \textbf{amt\_samples} & 600 & 600 & 600 \\ \hline
    \textbf{K\_beliefs} & 2500 & 2500 & 2500 \\ \hline
    \textbf{nrm\_lr} & 0.001 & 0.001 & 0.001 \\ \hline
    \textbf{perception\_lr} & 0.001 & 0.001 & 0.001 \\ \hline
    \textbf{dirichlet\_lr} & 0.01 & 0.01 & 0.01 \\ \hline
    \textbf{dirichlet\_iters} & 50 & 50 & 50 \\ \hline
    \textbf{dirichlet\_init} & 0.1 & 0.1 & 0.1 \\ \hline
    \textbf{dirichlet\_L2} & 900000 & 900000 & 900000 \\ \hline
    \textbf{nrm\_loss} & mse & mse & mse \\ \hline
    \textbf{hidden\_size} & 800 & 800 & 800 \\ \hline
    \textbf{layers} & 3 & 3 & 3 \\ \hline
    \textbf{prune} & False & False & True \\ \hline
    \textbf{predict\_only} & False & True & False \\ \hline
    \textbf{use\_prior} & True & True & True \\ \hline
    \end{tabular}
    \caption{Table of hyperparameters of A-Nesi for the variants explain, predict only, and prune.\label{tab:hyperparameters:anesi}}
\end{table}

\paragraph{Scallop.}
Since Scallop has been only applied to the MNIST-1-SUM task \cite{NeSy:Scallop:Huang2021}, we extend the experiments with Scallop to the MNIST-N-SUM task for N=\{2,3,4\} in order to use it as baseline. 
The logic programs are given in Listing~\ref{lst:scallop_rules}.
We use Scallop with exact as well as with approximate reasoning. 
For approximate reasoning of Scallop we choose the \texttt{diff-top-k-proof} semiring with $k=3$, as recommended in \cite{NeSy:Scallop:Huang2021}. 
To investigate the performance of Scallop with exact reasoning, we choose $k$ in a way that the logic provenance formula is in any case smalle than $k$ and is consequently not pruned. 
In particular, we choose $k$ for every MNIST-N-SUM task depending on N: 
$$k = 2^{2 \cdot 10 \cdot N}.$$ 
This way, $k$ represents an upper bound for the length of the logical provenance formula. 

\paragraph{DPLA$^*$. } 
DPLA$^*$ employs heuristics to estimate the probability of partial proofs to search for the best proof. 
Here, we use the geometric mean heuristic \cite{NeSy:DeepProbLog:Manhaeve2021}.

\begin{lstlisting}[caption={Rules used in Scallop for the MNIST-N-SUM task}, label={lst:scallop_rules}]
# N=1
self.scl_ctx.add_rule("summand_one(a) :- digit_1(a)")
self.scl_ctx.add_rule("summand_two(a) :- digit_2(a)")
# N=2
self.scl_ctx.add_rule("summand_one(10 * a + b) :- digit_1(a), digit_2(b)")
self.scl_ctx.add_rule("summand_two(10 * a + b) :- digit_3(a), digit_4(b)")
# N=3
self.scl_ctx.add_rule("summand_one(10 * 10 * a + 10 * b + c) :- digit_1(a), digit_2(b), digit_3(c)")
self.scl_ctx.add_rule("summand_two(10 * 10 * a + 10 * b + c) :- digit_4(a), digit_5(b), digit_6(c)")
# N=4
self.scl_ctx.add_rule("summand_one(10 * 10 * 10 * a + 10 * 10 * b + 10 * c + d) :- digit_1(a), digit_2(b), digit_3(c), digit_4(d)")
self.scl_ctx.add_rule("summand_two(10 * 10 * 10 * a + 10 * 10 * b + 10 * c + d) :- digit_5(a), digit_6(b), digit_7(c), digit_8(d)")
\end{lstlisting}

\end{document}

%% file: markup.tex
\newcommand{\tribe}[0]{\mathcal{A}}
\newcommand{\prob}[0]{\mathbb{P}}
\newcommand{\inputspace}[0]{\mathcal{I}}
\newcommand{\outputspace}[0]{\mathcal{O}}
\newcommand{\function}[0]{F}
\newcommand{\distrib}[0]{p}

\newcommand{\nn}[0]{\text{NN}}
\newcommand{\Functions}[2]{\mathcal{F}(#1,#2)}
\newcommand{\Vset}[1]{V_{#1}}
\newcommand{\Xvar}[1]{X_{#1}}
\newcommand{\Sfunction}[0]{S}

\newcommand{\unknown}[0]{\mathbf{unknown}}

\newcommand{\setofvaluations}{\mathcal{V}}